\documentclass{article} 
\usepackage{iclr_conference,times}

\usepackage{amsmath,amsfonts,bm}

\def\eqref#1{equation~\ref{#1}}

\def\1{\bm{1}}

\DeclareMathAlphabet{\mathsfit}{\encodingdefault}{\sfdefault}{m}{sl}
\SetMathAlphabet{\mathsfit}{bold}{\encodingdefault}{\sfdefault}{bx}{n}

\usepackage{hyperref}
\usepackage{url}
\usepackage{amsthm}
\usepackage{amssymb}
\usepackage{algorithm}
\usepackage{algpseudocode}
\usepackage{graphicx}
\usepackage{enumitem}
\usepackage{subcaption}

\usepackage{booktabs}
\usepackage[section]{placeins}

\newtheorem{theorem}{Theorem}

\newtheorem{proposition}{Proposition}
\newtheorem{definition}{Definition}
\newtheorem{lemma}{Lemma}

\newtheorem{corollary}{Corollary}

\title{From Adam to Adam-Like Lagrangians: Second-Order Nonlocal Dynamics}

\author{
Carlos Heredia\thanks{e-mail address: carlosherediapimienta@gmail.com} \\
\textit{IAMM Research, Department of Applied Artificial Intelligence} \\
\textit{DAMM, Carrer del Rossell\'o 515, 08025 Barcelona, Catalonia, Spain}
}

\iclrfinalcopy
\begin{document}

\maketitle

\begin{abstract}
In this paper, we derive an accelerated continuous-time formulation of Adam by modeling it as a second-order integro-differential dynamical system. We relate this inertial nonlocal model to an existing first-order nonlocal Adam flow through an $\alpha$-refinement limit, and we provide Lyapunov-based stability and convergence analyses. We also introduce an Adam-inspired nonlocal Lagrangian formulation, offering a variational viewpoint. Numerical simulations on Rosenbrock-type examples show agreement between the proposed dynamics and discrete Adam. 
\end{abstract}


\section{Introduction}

Adaptive optimization methods such as Adam \cite{kingma2014} are widely used in modern machine learning. Despite their empirical success, a systematic dynamical interpretation of their mechanisms remains comparatively less developed. The present work contributes to this direction by viewing Adam as a nonlocal dynamical system.

We consider the unconstrained optimization problem
\[
\min_{\theta \in \mathbb{R}^n} f(\theta),
\]
where $f:\mathbb{R}^n \to \mathbb{R}$ is differentiable and typically non-convex. A classical approach to approximate a minimizer is gradient descent \cite{Rumelhart1986}, which iteratively updates the parameters by moving in the direction of steepest decrease,
\[
\theta_{k+1}=\theta_k-\alpha\,\nabla f(\theta_k).
\]
Beyond its algorithmic role, gradient descent provides a canonical bridge between optimization and dynamical systems: under the time rescaling $t\approx k\alpha$, letting $\alpha\to 0$ turns the iteration into an explicit Euler discretization of the gradient flow
\begin{equation}\label{ref:gradient_flow}
\dot{\theta}(t)=-\nabla f(\theta(t)),    
\end{equation}
see, e.g., \cite{boyd2004}.

While Stochastic Gradient Descent (SGD) extends gradient descent by using mini-batches---randomly selected subsets of the data---to approximate gradients \citep{Bottou2010}, in this article we adopt a deterministic, full-gradient viewpoint. This choice aligns naturally with the continuous-time limit and allows for a more direct use of differential-equation techniques, which we exploit in Section~\ref{sec:stability}. For continuous-time descriptions of stochastic gradient methods, we refer to \citet{sirignano2017stochastic, lugosi2025}.

Although the gradient flow in \eqref{ref:gradient_flow} provides a clean dynamical picture, its practical performance often depends critically on the choice of learning rate: a single global stepsize must simultaneously ensure stability and make rapid progress, and in modern nonconvex landscapes this tuning can be delicate. This motivates adaptive methods that self-adjust the effective learning rate using information from past gradients. Adam is a prominent example: it introduces auxiliary moment variables $(m_k,v_k)$ that track exponentially weighted histories of gradients and squared gradients, thereby producing a direction that is smoothed in time and a step size that is normalized using an exponential moving average of squared gradients. In a schematic form, the update reads
\[
\theta_{k+1}
=\theta_k-\alpha\,\frac{\widehat m_k}{\sqrt{\widehat v_k}+\varepsilon},
\]
where $\widehat m_k$ and $\widehat v_k$ denote the usual bias-corrected moment estimates.
Consequently, the effective direction and step size at iteration $k$ are not determined by $\theta_k$ alone, but by a compressed summary of the past gradient signal, with memory horizons controlled by the averaging parameters $(\beta_1,\beta_2)$. Accordingly, a faithful continuous-time description of Adam should be nonlocal in time, with forcing terms that depend on the full past trajectory through causal memory kernels.

A closely related approach was proposed in \cite{heredia2025modelingadagradrmspropadam}, where adaptive methods (including Adam) are modeled directly through first-order integro-differential equations that preserve the intrinsic memory effects via nonlocal terms. That work establishes stability and convergence results for the resulting nonlocal flows and supports the continuous-time formulation with numerical simulations showing close agreement with the corresponding discrete algorithms.

Building on the first-order nonlocal (integro-differential) model for Adam, our aim here is to develop a refined dynamical viewpoint that remains faithful to the algorithm's memory mechanism while capturing higher-order effects.  We derive an accelerated continuous-time description in which Adam dynamics takes the form of an inertial system with damping and a genuinely nonlocal forcing term, reflecting the accumulated gradient history through causal memory operators. We then relate this accelerated model to the first-order nonlocal flow by showing that it reduces to the first-order description in the small-stepsize regime, up to controlled corrections away from the initial time.  On the analytical side, we provide Lyapunov-based stability and convergence upper bounds under standard assumptions and a natural compatibility condition on the averaging parameters (aligned with the classical Adam stability regime). Finally, we highlight a complementary variational perspective through Adam-inspired nonlocal Lagrangians, which suggests symmetry-based principles for algorithm design, and we support the theory with numerical experiments that illustrate both the stable and unstable parameter regimes and the qualitative behavior introduced by the additional initial condition, namely, the initial velocity.

Although we focus on Adam due to its widespread use, the same nonlocal dynamical and variational program can be carried out for other adaptive methods---including AdaGrad and RMSProp---suggesting a broader framework for understanding and designing adaptive optimizers.

\paragraph{Main contributions.}
The primary contributions of this work can be summarized as follows:
\begin{itemize}
  \item We derive a second-order continuous-time representation of Adam as an inertial equation with linear friction and a nonlocal forcing term, together with explicit causal convolution formulas for the continuous moments.
  \item We prove that, as $\alpha\to 0$, the accelerated second-order nonlocal dynamics recovers the first-order nonlocal Adam model on compact horizons away from the initial time, with quantitative asymptotic estimates.
  \item We establish well-posedness, positivity of the second moment, and Lyapunov-type stability and convergence upper bounds for $\beta_1 \le \sqrt{\beta_2}$, including Polyak–Łojasiewicz (PL)---\cite{PL}---and Kurdyka–Łojasiewicz (KL)---\cite{KL}---rate statements up to $\alpha$-dependent remainders.
  \item Motivated by the inertial nonlocal structure, we propose an Adam-like nonlocal Lagrangian formalism for second-order dynamics with memory and discuss the structural constraints induced by strict causality, highlighting links to symmetry-based design principles.
  \item We present numerical simulations that validate the modeling choices, compare discrete Adam with its continuous (first- and second-order) counterparts, diagnose instability outside the stable $\beta$-regime, and study the effect of the additional initial-velocity condition.
\end{itemize}

\paragraph{Organization.}
Section~\ref{Sec:Adam-dynamics} derives the second-order integro-differential formulation and the causal convolution representation of the moments, and establishes its $\alpha$-refinement link to the first-order nonlocal model. Section~\ref{sec:stability} develops stability and Lyapunov-based convergence bounds (including PL/KL rate statements). Section~\ref{sec:Lagrangian} introduces the nonlocal Lagrangian viewpoint and its connection to Adam-like dynamics under reciprocity constraints. Section~\ref{sec:Numerical_Simulations} presents numerical methods and experiments comparing the discrete and (first- and second-order) continuous dynamics.

\subsection{Related work}
Adam is among the most widely used adaptive optimization methods in deep learning due to its practical efficiency and robustness. At the same time, since Adam is intrinsically a discrete-time procedure, several works have pursued continuous-time limits and surrogate dynamical models in order to obtain sharper theoretical insight and to leverage analytical tools from continuous dynamical systems.

\paragraph{Non-convergence and memory effects:}
Despite its strong empirical performance, Adam can exhibit subtle failure modes tied to its exponential-memory mechanism. Because Adam relies on exponentially weighted moving averages of gradients and squared gradients, its adaptive preconditioning is effectively governed by a finite memory horizon controlled by $(\beta_1,\beta_2)$, which may downweight older (yet informative) gradient information. In the online convex optimization (OCO) setting, \cite{reddi2019convergenceadam} provide explicit instances where this short-memory behavior compromises convergence guarantees and can yield an average regret that does not vanish asymptotically\footnote{For a concise introduction to online convex optimization and regret, see \cite{hazan2023introductiononlineconvexoptimization}.}. These observations motivated variants that enforce more persistent accumulation of second-moment information, such as AMSGrad.

\paragraph{Continuous-time representations and stability (Lyapunov) in optimization.}
A growing line of research studies discrete optimization algorithms through their continuous-time limits, typically expressed as ordinary differential equations (ODEs). \cite{Su2016} establish a principled connection between Nesterov's accelerated gradient method and a second-order
ODE, using this viewpoint to explain its convergence behavior and trajectory dynamics. Complementarily, \cite{wibisono2016} derive accelerated dynamics from a variational (Lagrangian) formulation, clarifying the role of momentum as an intrinsic geometric ingredient. In the same spirit, classical momentum methods such as Polyak's Heavy Ball \cite{polyak1964speedup-Hball} can be interpreted via a damped
second-order system of the form
\begin{equation}
\ddot{x}(t) + \gamma \dot{x}(t) + \nabla f(x(t)) = 0,
\end{equation}
which provides a mechanical intuition (inertia plus friction) and serves as a useful baseline for comparing accelerated and momentum-based schemes \cite{Alvarez2000OnTM, ALVAREZ2002747}.

These continuous formulations enable a stability analysis of optimization dynamics using Lyapunov theory~\cite{gantmakher1970lectures,khalil2002}: by constructing a Lyapunov function---often combining objective suboptimality and a kinetic/momentum term \cite{Alvarez2000OnTM, ALVAREZ2002747, JMLRBelotto}---one can certify (asymptotic) convergence to equilibria and quantify robustness to perturbations. This perspective is particularly valuable when moving beyond classical momentum into adaptive methods. For example, \cite{JMLRBelotto} propose an ODE model for first-order adaptive algorithms (including Adam) in an augmented state space $(\theta,m,v)$, and leverage Lyapunov-type arguments to characterize convergence in continuous time. Their construction also provides a principled bridge between discrete hyperparameters and continuous-time parameters (e.g., via exponential relations between $\beta_i$ and time constants), supporting an interpretation of Adam as a forward-Euler discretization of an underlying ODE. Relatedly, \cite{Barakat2020} propose a continuous-time non-autonomous ODE model for Adam (capturing, in particular, the time-dependence induced by debiasing) and show that the piecewise-linear interpolation of Adam converges weakly to the ODE solution.  Using an explicit Lyapunov function, they establish well-posedness and convergence to critical points, and under a KL property they obtain convergence to a single critical point together with rates depending on the KL-exponent. 

Additionally, \cite{romero2022} address the nontrivial gap between continuous-time analyses and discrete algorithms: discretizing a continuous-time analogue of a gradient-based method does not, in general, preserve the Lyapunov structure or the associated stability guarantees. They therefore focus on designing structure-preserving discretizations of continuous-time optimization dynamics, aiming to retain---at the discrete level---the optimality properties implied by the underlying continuous formulation.

Viewed more broadly, these continuous-time and Lyapunov perspectives connect naturally to a complementary line of work that injects physical structure—notably variational principles, geometry, and conservation laws—directly into learning and dynamical modeling. Departing from standard neural architectures that learn input-output maps, Hamiltonian Neural Networks (HNNs)~\cite{greydanus2019hamiltonian} and Lagrangian Neural Networks (LNNs)~\cite{cranmer2020lagrangian} parameterize underlying energy function or functional action so that the resulting dynamics obey principled constraints: HNNs learn a Hamiltonian whose induced flow preserves invariants (making them well suited to energy-conserving systems), while LNNs learn a Lagrangian and recover dynamics via the Euler-Lagrange equations without requiring canonical coordinates or explicit momenta. Both approaches align with the philosophy of physics-informed neural networks (PINNs) \cite{PINNS}.

Building on the same geometric viewpoint, Noether’s theorem—linking symmetries to conserved quantities—has motivated architectures that meta-learn invariants (Noether Networks)~\cite{alet2021noether}, and has also inspired theoretical analyses of learning dynamics through symmetry and explicit symmetry breaking~\cite{tanaka2021noether}.

\paragraph{Integro-differential equations and nonlocal operators.}
Integro-differential equations (IDEs) provide a natural continuous-time language for systems with memory: the instantaneous dynamics depend on a weighted accumulation of past states through integral terms (often of Volterra type \cite{Tricomi1985IntegralEquations}, i.e., causal kernels). This viewpoint has recently been adopted in machine learning through neural parametrizations of nonlocal dynamics; for instance, Neural Integro-Differential Equations combine learnable function approximators with integral operators to model history-dependent evolution and improve extrapolation across temporal regimes and unseen initial conditions \cite{zappala2022neural}. Closely related ideas also appear in optimization: many adaptive methods are built from exponential moving averages, which in continuous time admit explicit causal convolution (kernel)
representations. In particular, \cite{gould2024} derive a continuous-time formulation of Adam/AdamW in which the first and second moment estimators can be written as history integrals with exponentially decaying kernels, making the update rule an explicit ratio of causal memory operators. Leveraging this integral (kernel) representation, they further characterize stable regions in the hyperparameter space $\mathcal{B}_+$ that guarantee bounded updates, providing a principled explanation of when the optimizer remains well behaved during training.
More broadly, nonlocal calculus develops optimization frameworks where classical differential operators are replaced by integral-operator analogues (nonlocal gradients) governed by interaction kernels \cite{SriramNonlocal}; this offers both modeling flexibility and a principled route to handle irregular/nonsmooth settings, while recovering the local (ODE/PDE) limit as kernels concentrate\footnote{By ``kernels concentrate'' we mean that \(\{K_\alpha\}\) forms an approximate identity (in distributions) as \(\alpha\to\infty\), i.e., \(K_\alpha \rightharpoonup \delta_0\) where $\delta_0$ denotes the Dirac delta distribution at \(0\)  \cite{Vladimirov_GF}.}. 

From a practical perspective, \cite{buades_denoising} introduce the Non-Local Means (NL-means) denoiser, which estimates each pixel as a normalized weighted average of other pixels (often restricted in practice to a finite search window), where weights are computed from patch similarity, commonly via an exponential of a Gaussian-weighted $\ell_2$ distance between local neighborhoods. This kernel-based viewpoint is closely related to the nonlocal calculus formalized by \cite{nonlocal_image}, where a similarity kernel \(w(x,y)\) induces nonlocal gradient, divergence, and Laplacian operators, leading to nonlocal TV/ROF-type variational models and PDE flows that separate distant but similar regions and often improve the handling of textures and repeated patterns.

\vspace{0.5cm}
While much of the continuous-time literature models optimization algorithms through ODE/PDE surrogates, such representations can obscure the history dependence intrinsic to adaptive methods based on exponential moving averages. Building on the first-order nonlocal (integro-differential) framework introduced in \citet{heredia2025modelingadagradrmspropadam}, we extend this viewpoint to an accelerated, second-order (inertial) setting, where the dynamics include damping and a genuinely nonlocal forcing term induced by causal memory operators. Specifically, we formulate Adam in continuous time as an integro-differential equation (IDE) in which the influence of past gradients enters explicitly through causal integral operators.

Concretely, our second-order IDE framework (a) represents cumulative effects by modeling how the entire past gradient history contributes to the current update through causal kernels; (b) enables analytical tools specific to IDEs and Volterra operators---including Lyapunov functionals---to establish well-posedness and convergence properties; (c) provides a principled bridge between discrete and continuous dynamics, clarifying how hyperparameters control effective memory horizons and suggesting design choices that can improve stability and robustness beyond what is apparent from purely ODE-based limits; and (d) admits a complementary variational formulation via Adam-inspired nonlocal Lagrangians, offering structural insight and symmetry-guided principles for algorithm design.

\section{Adam as a second-order integro-differential equation}\label{Sec:Adam-dynamics}

In this section we recast Adam as a continuous-time dynamical system. We start from the standard discrete formulation summarized in Algorithm~\ref{alg:Adam}. Given a differentiable objective function $f:\mathbb{R}^n\to\mathbb{R}$, the method evolves parameters $\theta_k\in\mathbb{R}^n$ together with two auxiliary states: the first moment $m_k\in\mathbb{R}^n$ and the second moment $v_k\in\mathbb{R}^n$ (or coordinatewise $m_k^i$ / $v_k^i$, depending on the convention). The gradient $g_k$ is evaluated at $\theta_{k-1}$, and $(m_k,v_k)$ are updated via exponential moving averages with decay factors $\beta_1,\beta_2\in[0,1)$. Bias-corrected versions $(\hat m_k,\hat v_k)$ are then used to define a preconditioned update of $\theta_k$ with stepsize $\alpha>0$ and stabilizer $\epsilon>0$.

\begin{algorithm}
\caption{Adam Optimization Algorithm}
\label{alg:Adam}
\begin{algorithmic}[1] 
    \State Initialize the parameters $\theta_{k=0}$, the first moment $m_{k=0} = 0$, and the second moment $v_{k=0} = 0$.
    \While{the parameters $\theta_k$ have not converged}
        \State $k \leftarrow k + 1$
        \State Compute the gradient: $g_{k} \leftarrow \nabla f(\theta_{k-1})$.
        \State Update first moment: $m_{k} \leftarrow \beta_1 m_{k-1} + (1 - \beta_1) g_{k}$
        \State Update second moment: $v_{k} \leftarrow \beta_2 v_{k-1} + (1 - \beta_2) g_{k}^2$
        \State Correct first moment: $\hat{m}_{k} \leftarrow m_{k}/(1 - \beta_1^k)$
        \State Correct second moment: $\hat{v}_{k} \leftarrow v_{k}/(1 - \beta_2^k)$
        \State Update parameters: $\theta_{k} \leftarrow \theta_{k-1} - \alpha\, \hat{m}_{k}/(\sqrt{\hat{v}_{k}} + \epsilon)$
    \EndWhile
\end{algorithmic}
\end{algorithm}

While the original Adam formulation in \cite{kingma2014} maintains the second moment coordinatewise,
\[
v_k = \beta_2 v_{k-1} + (1-\beta_2)\,(g_k \odot g_k), \qquad v_k\in\mathbb{R}^n,
\]
thus yielding a diagonal preconditioner, we will adopt an isotropic version for the purposes of the continuous-time limit. Concretely, we take $v_k\in\mathbb{R}$ to track a global second-moment estimate, and we replace the diagonal scaling by a scalar factor $(\sqrt{\hat v_k}+\epsilon)^{-1}$ multiplying the update direction $\hat m_k$. This choice keeps the structure of Adam while avoiding componentwise effects that are not essential for the integro-differential formulation developed below.

The key observation for our purposes is that Adam is not merely a first-order recursion in \(\theta_k\): it is a memory-driven system, since \((m_k,v_k)\) encode a filtered history of gradients along the trajectory. This structure makes it natural to investigate the regime of small stepsizes \(\alpha\) through a continuous-time limit \(t\approx k\alpha\). Moreover, to reveal the inertial character that is hidden in the one-step recursion, we expand the parameter update to second order in \(\alpha\). The precise statement---including the resulting second-order integro-differential dynamics---is given by the following Proposition:

\begin{proposition}\label{prop:Adam2}
Let $\alpha>0$ be small and consider the continuous limit $t\approx k\,\alpha$ with the second-order time expansion
\[
\theta^i_{k+1} \to \theta^i(t+\alpha)\;=\;\theta^i(t)+\alpha\,\dot\theta^i(t)+\tfrac{\alpha^2}{2}\,\ddot\theta^i(t)\,+\mathcal{O}(\alpha^3).
\]
With the initial data $m^i(0)=\dot m^i(0)=0$ and $v(0)=\dot v(0)=0$, the Adam update admits the inertial ODE with linear friction
\begin{equation}\label{eq:inertial}
\frac{\alpha}{2}\,\ddot\theta^i(t)+\dot\theta^i(t)\;=\;-\eta(t+\alpha)\,T^i(\theta,t+\alpha),
\qquad
T^i(\theta,t)\;:=\;\frac{m^i(\theta,t)}{\sqrt{v(\theta,t)}+\varepsilon(t)}, 
\end{equation}
where the continuous bias-correction factors are
\[
\eta(t)\;:=\;\frac{\sqrt{\,1-\beta_2^{\,t/\alpha}\,}}{\,1-\beta_1^{\,t/\alpha}\,},\qquad
\varepsilon(t)\;:=\;\epsilon\,\sqrt{\,1-\beta_2^{\,t/\alpha}\,}.
\]
The continuous moments are causal convolutions
\[
m^i(\theta,t)\;=\;\frac{1-\beta_1}{\alpha}\int_{0}^{t} K_{\beta_1}(t-\tau)\;\partial^i f\!\big(\theta(\tau)\big)\,d\tau,\quad
v(\theta,t)\;=\;\frac{1-\beta_2}{\alpha}\int_{0}^{t} K_{\beta_2}(t-\tau)\;\|\nabla f(\theta(\tau))\|^2\,d\tau,
\]
with kernels $K_\beta(s)$ (for $s:=t-\tau\ge 0$) given case-wise by
\begin{itemize}
\item[\textbullet] Critical damping $\beta=\tfrac12$:
\[
K_{1/2}(s)=\frac{2\,s}{\alpha}\,e^{-s/\alpha}.
\]
\item[\textbullet] Overdamped regime $\beta>\tfrac12$. With $\kappa:=\sqrt{\,2\beta-1\,}$,
\[
K_\beta(s)\;=\; \frac{2}{\kappa}\,e^{-s/\alpha}\,\sinh\!\Big(\frac{\kappa}{\alpha}s\Big)\,.
\]
\item[\textbullet] Underdamped regime $\beta<\tfrac12$. With $\rho:=\sqrt{\,1-2\beta\,}$,
\[
K_\beta(s)=\frac{2}{\rho}\,e^{-s/\alpha}\,\sin\!\Big(\frac{\rho}{\alpha}s\Big).
\]
\end{itemize}
\end{proposition}

\begin{proof}
Let $\alpha>0$ be the step size and set $t=\alpha k$ as a continuous-time reparametrization. For any sufficiently smooth trajectory $z(t)$ we use the second-order Taylor expansion
\[
z_{k+1} \to z(t+\alpha)=z(t)+\alpha \dot z(t)+\tfrac{\alpha^2}{2}\ddot z(t)+\mathcal{O}(\alpha^3).
\]
We start from Adam’s recurrences and reindex $k\mapsto k+1$ to avoid the term $f(\theta_{k-1})$:
\[
m^i_{k+1}=\beta_1 m^i_k+(1-\beta_1)\,g^i_{k+1},
\qquad
v_{k+1}=\beta_2 v_k+(1-\beta_2)\,g_{k+1}^2.
\]
Here $g^i_{k+1}=\partial^i f(\theta_k)$ and $g_{k+1}^2=\|\nabla f(\theta_k)\|^2$. Under the time rescaling $t=\alpha k$, a Taylor expansion up to $O(\alpha^3)$ gives the following linear ODEs with constant coefficients:
\[
\tfrac{\alpha^2}{2}\,\ddot m^i+\alpha\,\dot m^i+(1-\beta_1)\,m^i=(1-\beta_1)\,\partial^i f(\theta),
\qquad
\tfrac{\alpha^2}{2}\,\ddot v+\alpha\,\dot v+(1-\beta_2)\,v=(1-\beta_2)\,\|\nabla f(\theta)\|^2,
\]
with initial data $m^i(0)=\dot m^i(0)=0$ and $v(0)=\dot v(0)=0$. Since both share the same left-hand operator, it suffices to analyze the scalar problem
\[
\tfrac{\alpha^2}{2}\,y''+\alpha\,y'+(1-\beta)\,y=(1-\beta)\,r(t),
\]
for $0\leq\beta<1$ and a given source $r$. The characteristic polynomial of the homogeneous equation is
\[
\tfrac{\alpha^2}{2}\,\omega^2+\alpha\,\omega+(1-\beta)=0,
\]
with roots
\[
\omega_\pm=\frac{-1\pm\sqrt{2\beta-1}}{\alpha}.
\]
For all $0\le \beta<1$, both characteristic roots satisfy $\Re(\omega_\pm)<0$, hence the homogeneous dynamics is exponentially stable. The regimes $\beta=\tfrac12$, $\beta>\tfrac12$, and $\beta<\tfrac12$ correspond to critical, over-, and under-damping, respectively.

Dividing by $\alpha^2/2$ gives the normalized form
\[
y''+\frac{2}{\alpha}\,y'+\frac{2(1-\beta)}{\alpha^2}\,y
=F(t),\qquad F(t):=\frac{2(1-\beta)}{\alpha^2}\,r(t).
\]
The homogeneous equation admits fundamental solutions $y_1(t)=e^{\omega_+ t}$ and $y_2(t)=e^{\omega_- t}$ with Wronskian
\[
W(t)=(\omega_-\omega_+)\,e^{(\omega_++\omega_-)t}=(\omega_-\omega_+)\,e^{-2t/\alpha}.
\]
Variation of constants yields
\[
y(t)=\int_{0}^{t}G_\beta(t-\tau)\,F(\tau)\,d\tau,\qquad
G_\beta(s)=\frac{e^{\omega_+ s}-e^{\omega_- s}}{\omega_+-\omega_-}\,\mathbf 1_{\{s\ge 0\}}.
\]
Undoing the normalization (i.e., replacing $F$ by $\tfrac{2(1-\beta)}{\alpha^2}r$) we obtain the convolution representation
\[
y(t)=\frac{1-\beta}{\alpha}\int_{0}^{t}K_\beta(s)\,r(t-s)\,ds,\qquad
K_\beta(s)=\frac{2}{\alpha\,(\omega_+-\omega_-)}\Big(e^{\omega_+ s}-e^{\omega_- s}\Big)\mathbf 1_{\{s\ge 0\}}.
\]
Using $\omega_++\omega_-=-2/\alpha$ and
\[
\omega_+-\omega_-=
\begin{cases}
\frac{2\kappa}{\alpha}, & \beta>\tfrac12,\ \kappa:=\sqrt{2\beta-1},\\[2pt]
0, & \beta=\tfrac12,\\[2pt]
\frac{2i\rho}{\alpha}, & \beta<\tfrac12,\ \rho:=\sqrt{1-2\beta},
\end{cases}
\]
we obtain the closed forms
\[
K_\beta(s)=
\begin{cases}
\displaystyle \frac{2}{\kappa}\,e^{-s/\alpha}\,\sinh\!\Big(\frac{\kappa}{\alpha}s\Big), & \beta>\tfrac12,\\[8pt]
\displaystyle \frac{2\,s}{\alpha}\,e^{-s/\alpha}, & \beta=\tfrac12,\\[6pt]
\displaystyle \frac{2}{\rho}\,e^{-s/\alpha}\,\sin\!\Big(\frac{\rho}{\alpha}s\Big), & \beta<\tfrac12,
\end{cases}
\qquad s\ge 0,
\]
in particular $K_\beta(0)=0$.

Applying this to the equations for $m^i$ and $v$ (with $r^i(t)=\partial^i f(\theta(t))$ and $r(t)=\|\nabla f(\theta(t))\|^2$, respectively) gives
\[
m^i(\theta, t)=\frac{1-\beta_1}{\alpha}\int_{0}^{t}K_{\beta_1}(t-\tau)\,\partial^i f(\theta(\tau))\,d\tau,
\qquad
v(\theta, t)=\frac{1-\beta_2}{\alpha}\int_{0}^{t}K_{\beta_2}(t-\tau)\,\|\nabla f(\theta(\tau))\|^2\,d\tau,
\]
and, since $K_\beta(0)=0$, the causal initial conditions $m^i(0)=\dot m^i(0)=0$ and $v(0)=\dot v(0)=0$ are automatically satisfied.

For the parameter dynamics, define the (continuously bias-corrected) quantities
\[
T^i(\theta,t):=\frac{m^i(\theta,t)}{\sqrt{v(\theta,t)}+\varepsilon(t)},\qquad
\eta(t):=\frac{\sqrt{1-\beta_2^{\,t/\alpha}}}{1-\beta_1^{\,t/\alpha}},\qquad
\varepsilon(t):=\epsilon\,\sqrt{1-\beta_2^{\,t/\alpha}}.
\]
The second-order expansion of the update
\[
\theta^i_{k+1}- \theta^i_{k} \;=\;-\alpha\,\eta(t+\alpha)\,T^i(t+\alpha)
\;\approx\;\alpha\,\dot\theta^i(t)+\tfrac{\alpha^2}{2}\,\ddot\theta^i(t)
\]
implies, after rearranging,
\[
\frac{\alpha}{2}\,\ddot\theta^i(t)+\dot\theta^i(t)=\,-\eta(t+\alpha)\,T^i(\theta,t+\alpha)\,.
\]
\end{proof}

Let us point out the time domain on which the integro-differential formulation is naturally posed. In Proposition~\ref{prop:Adam2} the forcing term is evaluated at the shifted time $t+\alpha$, which directly reflects the one-step structure of Algorithm~\ref{alg:Adam}. Indeed, the update from $k-1$ to $k$ uses the moments computed at iteration $k$, and under the scaling $t\approx k\alpha$ this corresponds to expressing the right-hand side at time $t+\alpha$. Equivalently, setting $s:=t+\alpha$, we obtain
\[
\frac{\alpha}{2}\,\ddot\theta^i(s-\alpha)+\dot\theta^i(s-\alpha)\;=\;-\,\eta(s)\,T^i(\theta,s),
\qquad s\ge \alpha,
\]
so that the nonlocal forcing term (and hence the integro-differential representation) is canonically defined for $s\in[\alpha,\infty)$, in agreement with the discrete dynamics in which the first nontrivial moment update occurs after one iteration.

This representation also makes transparent two basic qualitative features of the dynamics. First, observe that for any $\beta\in[0,1)$ the kernels $K_\beta(s)$ share the common factor $e^{-s/\alpha}$. Consequently, contributions from past gradients are exponentially down-weighted in time.

Second, in the discrete Adam scheme the second-moment iterate satisfies $v_k\ge 0$ for all $k$, since it is constructed as an exponentially weighted average of squared gradients. In the continuous-time formulation, it is therefore natural to require the analogous property $v(t)\ge 0$. This is immediate in the critically damped and overdamped regimes, where $K_\beta(s)\ge 0$ for $s\ge 0$. By contrast, in the underdamped case $\beta<\tfrac12$ the kernel $K_\beta$ is oscillatory and changes sign, so $v(t)$ is no longer manifestly nonnegative. Next, we provide simple a sufficient condition ensuring $v(t)\ge 0$ in this regime.

\begin{lemma}\label{lemma:vp}
Let $\alpha>0$ and $\beta<\tfrac12$. Define
\[
\rho_\beta:=\sqrt{1-2\beta}\in(0,1],\qquad 
q_\beta:=\exp\!\Big(-\frac{\pi}{\rho_\beta}\Big).
\]
Consider the kernel
\[
K_\beta(s):=\frac{2}{\rho_\beta}\,e^{-s/\alpha}\,\sin\!\Big(\frac{\rho_\beta}{\alpha}s\Big),\qquad s\ge 0,
\]
and let $K_+,K_-:[0,\infty)\to[0,\infty)$ denote its positive and negative parts,
\[
K_+(s):=\max\left\{K_\beta(s),0\right\},\qquad 
K_-(s):=\max\left\{-K_\beta(s),0\right\},
\]
so that $K_\beta=K_+-K_-$ and $K_\pm\ge 0$. For $t\ge 0$, define the truncated terms
\[
A_+(t):=\int_0^t K_+(s)\,ds,\qquad 
A_-(t):=\int_0^t K_-(s)\,ds.
\]
Let $r(\tau):=\|\nabla f(\theta(\tau))\|^2\ge 0$ and define
\[
v(t):=\frac{1-\beta}{\alpha}\int_0^t K_\beta(s)\,r(t-s)\,ds,\qquad t\ge 0.
\]
Assume that
\[
\underline r(t):=\inf_{0\le s\le t} r(t-s)\ \ge\ q_\beta\,\overline r(t),
\qquad 
\overline r(t):=\sup_{0\le s\le t} r(t-s).
\]
Then $v(t)\ge 0$ for all $t\ge 0$.
\end{lemma}

\begin{proof}
Since \(K_\beta=K_+-K_-\) with \(K_\pm\ge 0\), we have
\[
v(t)=\frac{1-\beta}{\alpha}\!\left(\int_0^t K_+(s)\,r(t-s)\,ds-\int_0^t K_-(s)\,r(t-s)\,ds\right)
\ge \frac{1-\beta}{\alpha}\big(\underline r(t)\,A_+(t)-\overline r(t)\,A_-(t)\big).
\]
Thus it suffices to show that, for all \(t\ge 0\),
\begin{equation}\label{eq:trunc-bound}
A_-(t)\ \le\ q_\beta\,A_+(t).
\end{equation}

First of all, with \(\delta=\pi\alpha/\rho_\beta\), for every \(s\ge 0\),
\[
K_\beta(s+\delta)
=\frac{2}{\rho_\beta}\,e^{-(s+\delta)/\alpha}\,\sin\!\Big(\tfrac{\rho_\beta}{\alpha}(s+\delta)\Big)
=\frac{2}{\rho_\beta}\,e^{-s/\alpha}e^{-\pi/\rho_\beta}\,\sin\!\Big(\tfrac{\rho_\beta}{\alpha}s+\pi\Big)
=-e^{-\pi/\rho_\beta}\,K_\beta(s).
\]
Pointwise this implies
\[
K_-(s+\delta)=e^{-\pi/\rho_\beta}\,K_+(s),\qquad K_+(s+\delta)=e^{-\pi/\rho_\beta}\,K_-(s).
\]
Now, if \(0\le t\le \delta\), then \(K_\beta\ge 0\) on \([0,\delta]\), so \(A_-(t)=0\) and the bound is trivial because $q_\beta \geq 0$ and $A_+(r)\geq 0$ and therefore $0 \leq q_\beta A_+(t)$. If \(t>\delta\),
\[
\begin{aligned}
A_-(t)&=\int_0^t K_-(u)\,du=\int_\delta^t K_-(u)\,du
=\int_0^{t-\delta} K_-(s+\delta)\,ds\\
&=e^{-\pi/\rho_\beta}\int_0^{t-\delta} K_+(s)\,ds
\le e^{-\pi/\rho_\beta}\int_0^{t} K_+(s)\,ds
= q_\beta\,A_+(t).  
\end{aligned}
\]
Hence~\eqref{eq:trunc-bound} holds for all \(t\ge 0\). Substituting into the bound for \(v(t)\) yields
\[
v(t)\ \ge\ \frac{1-\beta}{\alpha}\,A_+(t)\,\big(\underline r(t)-q_\beta\,\overline r(t)\big)\ \ge\ 0,
\]
whenever \(\underline r(t)\ge q_\beta\,\overline r(t)\).
\end{proof}

For completeness, note that when $\beta\ge \tfrac12$ the kernel is nonnegative (at $\beta=\tfrac12$ one has $K_\beta(s)=\tfrac{2s}{\alpha}e^{-s/\alpha}$, whereas for $\beta>\tfrac12$ the kernel admits the $\sinh$ representation), hence $A_-=0$ and $v(t)\ge 0$ for any $r\ge 0$. Therefore, under the assumptions of Lemma~\ref{lemma:vp} we obtain in particular
\[
v(t)\;\ge\;0\qquad\text{for all }t\ge 0.
\]
Throughout the paper we restrict to parameter regimes for which the second-moment kernel satisfies the hypotheses of Lemma~\ref{lemma:vp}. In particular, $v(t)\ge 0$ for all $t\ge 0$, and thus $\sqrt{v(t)}$ is well-defined.

A natural question at this point is whether the second-order nonlocal formulation can be understood as an $\alpha$-refinement of the first-order nonlocal Adam model analyzed in \cite{heredia2025modelingadagradrmspropadam}. In particular, we would like to clarify in which precise sense the second-order accelerated dynamics reduces to the first-order nonlocal dynamics as $\alpha \to 0$. Here, $\alpha\to0$ is interpreted as a refinement process within the continuous-time formulation, rather than as the discrete-to-continuous time-step limit considered for gradient descent in the Introduction.

The following proposition shows that, under mild regularity assumptions, the second-order dynamics recovers the first-order nonlocal Adam dynamics uniformly away from the initial time. For notational simplicity, throughout this proposition we work with $\theta$ (and the associated moment $m$) rather than component-wise notation $\theta^{i}$ (and $m^{i}$), as the estimates apply identically to each coordinate. Moreover, $\|\cdot\|$ denotes the standard Euclidean norm on $\mathbb{R}^d$. Accordingly, we present it as a quantitative comparison, both at the level of trajectories and at the level of the associated nonlocal moments:
\begin{proposition}\label{prop:alpha-recovery}
Let $T>0$, $0<\delta<T$, and $\alpha_0\in(0,\delta)$. For $a\in\{1,2\}$ define
\[
H^{(1)}_a(\tau):=\frac{1-\beta_a}{\alpha}\exp\!\Big(-\frac{1-\beta_a}{\alpha}\tau\Big),
\qquad
H^{(2)}_a(\tau):=\frac{1-\beta_a}{\alpha}\,K_{\beta_a}(\tau),
\qquad \tau\ge 0,
\]
where $K_\beta$ is given in Proposition~\ref{prop:Adam2}. Let $g=\nabla f$ and set, for $j\in\{1,2\}$,
\[
m^{(j)}(t):=\int_0^t H^{(j)}_1(\tau)\,g(\theta(t-\tau))\,d\tau,
\qquad
v^{(j)}(t):=\int_0^t H^{(j)}_2(\tau)\,\|g(\theta(t-\tau))\|^2\,d\tau .
\]

Assume $\theta\in C^3([0,T];\mathbb R^d)$ and that the maps $t\mapsto g(\theta(t))$ and $t\mapsto \|g(\theta(t))\|^2$ have bounded second derivatives on $[0,T]$.
Assume moreover that
\[
\sup_{t\in[\delta-\alpha_0,\,T]}\|\ddot\theta(t)\|\le C_T
\]
for some $C_T$ independent of $\alpha$. Then there exist $\varepsilon_0>0$ (depending on $\delta$ but not on $\alpha$) such that $\varepsilon(t)\ge \varepsilon_0$ for all $t\in[\delta,T]$. Furthermore, there exists $C>0$, independent of $\alpha$, such that for all $t\in[\delta,T]$,
\[
\|m^{(2)}(t)-m^{(1)}(t)\|
\le C\Big(\frac{\alpha}{1-\beta_1}\Big)^{2},
\qquad
\bigl|v^{(2)}(t)-v^{(1)}(t)\bigr|
\le C\Big(\frac{\alpha}{1-\beta_2}\Big)^{2}.
\]
Consequently, uniformly for $t\in[\delta,T]$,
\[
\frac{m^{(2)}(t)}{\sqrt{v^{(2)}(t)}+\varepsilon(t)}
=
\frac{m^{(1)}(t)}{\sqrt{v^{(1)}(t)}+\varepsilon(t)}
+\mathcal{O}\!\left(\max\left\{\Big(\tfrac{\alpha}{1-\beta_1}\Big)^2,\ \tfrac{\alpha}{1-\beta_2}\right\}\right).
\]

Finally, if $0<\alpha<\alpha_0$ and $\theta$ satisfies
\[
\frac{\alpha}{2}\ddot\theta(t)+\dot\theta(t)
=
-\eta(t+\alpha)\,
\frac{m^{(2)}(t+\alpha)}{\sqrt{v^{(2)}(t+\alpha)}+\varepsilon(t+\alpha)},
\]
then, uniformly for $t\in[\delta-\alpha,\,T-\alpha]$,
\[
\dot\theta(t)
=
-\eta(t+\alpha)\,
\frac{m^{(1)}(t+\alpha)}{\sqrt{v^{(1)}(t+\alpha)}+\varepsilon(t+\alpha)}
+\mathcal{O}(\rho),
\qquad
\rho:=\max\left\{\alpha,\Big(\tfrac{\alpha}{1-\beta_1}\Big)^2,\ \tfrac{\alpha}{1-\beta_2}\right\}.
\]
\end{proposition}

\begin{proof}
We work with the truncated causal convolution
\[
(H*\varphi)(t):=\int_0^t H(\tau)\,\varphi(t-\tau)\,d\tau.
\]
For $a\in\{1,2\}$ and $j\in\{1,2\}$, define the full moments
\[
M^{(j)}_{a,k}:=\int_0^\infty \tau^k\,H^{(j)}_a(\tau)\,d\tau,\qquad k=0,1,2,
\]
and the truncated moments
\[
M^{(j)}_{a,k}(t):=\int_0^t \tau^k\,H^{(j)}_a(\tau)\,d\tau,\qquad t\in[0,T].
\]

For $H^{(1)}_a(\tau)=\lambda_a e^{-\lambda_a\tau}$ with $\lambda_a=(1-\beta_a)/\alpha$,
\[
M^{(1)}_{a,0}=1,\qquad
M^{(1)}_{a,1}=\lambda_a^{-1},\qquad
M^{(1)}_{a,2}=2\,\lambda_a^{-2}.
\]
For $H^{(2)}_a(\tau)=\lambda_a\,K_{\beta_a}(\tau)$, one obtains using the closed forms in Proposition~\ref{prop:Adam2}:
\[
M^{(2)}_{a,0}=1,\qquad
M^{(2)}_{a,1}=\lambda_a^{-1},\qquad
M^{(2)}_{a,2}=(1+\beta_a)\,\lambda_a^{-2}.
\]

Since the kernels decay exponentially on the scale $\alpha$, there exist constants $c,C>0$
(independent of $\alpha$) such that for $k=0,1$ and all $t\in[\delta,T]$,
\[
\bigl|M^{(j)}_{a,k}(t)-M^{(j)}_{a,k}\bigr|
=\int_t^\infty \tau^k |H^{(j)}_a(\tau)|\,d\tau
\le C\,e^{-c\,t/\alpha}
\le C\,e^{-c\,\delta/\alpha}.
\]
In particular, since $M^{(1)}_{a,0}=M^{(2)}_{a,0}=1$ and $M^{(1)}_{a,1}=M^{(2)}_{a,1}=\lambda_a^{-1}$,
we have
\[
\bigl|M^{(1)}_{a,0}(t)-M^{(2)}_{a,0}(t)\bigr|
+\bigl|M^{(1)}_{a,1}(t)-M^{(2)}_{a,1}(t)\bigr|
\le C\,e^{-c\,\delta/\alpha},
\qquad t\in[\delta,T].
\]

Let $H^{(1)},H^{(2)}$ be causal kernels and let $\varphi\in C^3([0,T])$.
For each $t\in[0,T]$ and $0\le \tau\le t$, Taylor's theorem gives
\[
\varphi(t-\tau)=\varphi(t)-\tau \dot\varphi(t)+\frac{\tau^2}{2}\,\ddot \varphi(t) + \mathcal{O}(\tau^3).
\]
Hence, 
\begin{align*}
(H^{(1)}*\varphi)(t)-(H^{(2)}*\varphi)(t)
&=\varphi(t)\bigl(M^{(1)}_0(t)-M^{(2)}_0(t)\bigr)
-\dot\varphi(t)\bigl(M^{(1)}_1(t)-M^{(2)}_1(t)\bigr)\\
&\quad+\frac12 \ddot\varphi(t)\int_0^t \tau^2\Bigl(H^{(1)}(\tau)-H^{(2)}(\tau)\Bigr)\,d\tau + \mathcal{O}(\tau^3).
\end{align*}
Taking absolute values and using $\|\ddot\varphi\|_{L^\infty([0,T])}<\infty$ yields
\[
\bigl|(H^{(1)}*\varphi)(t)-(H^{(2)}*\varphi)(t)\bigr|
\le C\,e^{-c\,\delta/\alpha}
+\frac{\|\ddot\varphi\|_{L^\infty([\delta,T])}}{2}
\left(\int_0^\infty \tau^2|H^{(1)}(\tau)|\,d\tau+\int_0^\infty \tau^2|H^{(2)}(\tau)|\,d\tau\right),
\]
uniformly for $t\in[\delta,T]$. For the kernels $H^{(j)}_a$, the second-moment integrals scale like $\lambda_a^{-2}$, hence
\[
\bigl|(H^{(1)}*\varphi)(t)-(H^{(2)}*\varphi)(t)\bigr|
\le C\left(\frac{\alpha}{1-\beta_a}\right)^2 + C\,e^{-c\,\delta/\alpha},
\qquad t\in[\delta,T].
\]
For $\alpha$ sufficiently small (depending on $\delta$), the exponential term is dominated by the
polynomial term, so the bound simplifies to $C(\alpha/(1-\beta_a))^2$.

Now, we apply the previous estimate componentwise with $\varphi_i(t)=g_i(\theta(t))$ and with $\varphi(t)=\|g(\theta(t))\|^2$. By the assumptions, $\varphi_i''$ and $\varphi''$ are bounded on $[0,T]$. Using the bounds on $M^{(j)}_{a,2}$ we deduce
\[
\|m^{(2)}(t)-m^{(1)}(t)\|
=
\mathcal O\!\Big(\Big(\tfrac{\alpha}{1-\beta_1}\Big)^{\!2}\Big),
\qquad
|v^{(2)}(t)-v^{(1)}(t)|
=
\mathcal O\!\Big(\Big(\tfrac{\alpha}{1-\beta_2}\Big)^{\!2}\Big),
\]
uniformly on $[\delta,T]$.

Let $F(x,y)=x/(\sqrt{y}+\varepsilon)$. For $y\ge 0$, $y+\delta y \geq 0$ and $\varepsilon\ge\varepsilon_0$, we can deduce that:
\[
|F(x+\delta x,y+\delta y)-F(x,y)|
\le
\frac{|\delta x|}{\varepsilon_0}
+
\,\frac{|x|}{\varepsilon_0^2}\,\sqrt{|\delta y|}\,.
\]
Indeed,
\[
\frac{x+\delta x}{\sqrt{y+\delta y}+\varepsilon}-\frac{x}{\sqrt{y}+\varepsilon}
=\frac{\delta x}{\sqrt{y+\delta y}+\varepsilon} + x\left[\frac{1}{\sqrt{y+\delta y} + \varepsilon} - \frac{1}{\sqrt{y}+\varepsilon} \right]
\]
The first term is bounded by $|\delta x|/\varepsilon_0$. For the second, we use the fact that
\[
|\sqrt{y+\delta y}-\sqrt{y}| \leq \frac{|\delta y|}{\sqrt{y+\delta y} + \sqrt{y}}\leq \sqrt{|\delta y|}\,,
\]
therefore, it becomes bounded by
\[
\left|x\!\left(\frac{1}{\sqrt{y+\delta y}+\varepsilon}-\frac{1}{\sqrt{y}+\varepsilon}\right)\right| \leq \frac{|x|}{\varepsilon^2_0} \sqrt{|\delta y|}\,.
\]

Thus, taking $x=m^{(1)}(t)$, $y=v^{(1)}(t)$, $\delta x=m^{(2)}(t)-m^{(1)}(t)$ and $\delta y=v^{(2)}(t)-v^{(1)}(t)$, and assuming $v^{(j)}(t)\ge0$, we have\footnote{Note that $y+\delta y=v^{(2)}(t)\ge 0$ by assumption.}
\[
\delta x=\mathcal{O}\!\Big(\Big(\tfrac{\alpha}{1-\beta_1}\Big)^{\!2}\Big),\qquad
\delta y=\mathcal{O}\!\Big(\Big(\tfrac{\alpha}{1-\beta_2}\Big)^{\!2}\Big).
\]
Notice that, since $[0,T]$ is compact and $g\circ\theta$ is continuous, there exists
\[
G_T:=\sup_{s\in[0,T]}\|g(\theta(s))\|<\infty.
\]
Therefore, using $H^{(1)}_1\ge 0$ and $\int_0^t H^{(1)}_1(\tau)\,d\tau\ \leq 1$, we obtain
\[
\|m^{(1)}(t)\|
\le \int_0^t H^{(1)}_1(\tau)\,\|g(\theta(t-\tau))\|\,d\tau
\le G_T\int_0^t H^{(1)}_1(\tau)\,d\tau
\leq G_T,
\]
so that $|x|=\|m^{(1)}(t)\|$ is uniformly bounded and can be absorbed into the implicit constant. Hence,
\begin{equation}\label{eq:diff_moments_order}
  \frac{m^{(2)}(t)}{\sqrt{\,v^{(2)}(t)}+\varepsilon(t)\,}
-\frac{m^{(1)}(t)}{\sqrt{\,v^{(1)}(t)}+\varepsilon(t)\,}
=
\mathcal{O}\!\left(\max\left\{\Big(\tfrac{\alpha}{1-\beta_1}\Big)^2,\ \tfrac{\alpha}{1-\beta_2}\right\}\right), \qquad t\in[\delta, T].
\end{equation}

Finally, we rewrite the accelerated shifted equation as
\[
\dot\theta(t)
=
-\eta(t+\alpha)\,
\frac{m^{(2)}(t+\alpha)}{\sqrt{v^{(2)}(t+\alpha)}+\varepsilon(t+\alpha)}
-\frac{\alpha}{2}\ddot\theta(t).
\]
By assumption $\sup_{t\in[\delta-\alpha_0,\,T]}\|\ddot\theta(t)\|\le C_T$, hence
\[
\Big\|\tfrac{\alpha}{2}\ddot\theta(t)\Big\|\le \tfrac{\alpha}{2}C_T=\mathcal O(\alpha)
\quad\text{uniformly for }t\in[\delta-\alpha,\,T-\alpha],
\]
for all $0<\alpha<\alpha_0$ with $\alpha_0\in(0,\delta)$. Together with the kernel comparison---\eqref{eq:diff_moments_order}---, this yields, upon substitution into the accelerated dynamics,
\[
\dot\theta(t)=
-\eta(t+\alpha)\,
\frac{m^{(1)}(t+\alpha)}{\sqrt{v^{(1)}(t+\alpha)}+\varepsilon(t+\alpha)}
+\mathcal O(\rho),
\quad t\in[\delta-\alpha,T-\alpha]\,,
\]
where
\[
\rho:=\max\left\{\alpha,\Big(\tfrac{\alpha}{1-\beta_1}\Big)^2,\ \tfrac{\alpha}{1-\beta_2}\right\}\,.
\]
In particular, if $\beta_1,\beta_2$ are fixed and $\alpha\to0$, then $\left(\alpha/(1-\beta_1)\right)^2 = \mathcal O(\alpha^2)\subset \mathcal{O}(\alpha)$ and $\alpha/(1-\beta_2) = \mathcal O(\alpha)$, so the remainder can be written as $\mathcal O(\alpha)$.
\end{proof}

Let us comment on the content of this proposition. Its main message is that the second-order (inertial) formulation does not generate a genuinely new drift on fixed finite horizons independent of $\alpha$: away from the initial time it can be viewed as a perturbation of the first-order nonlocal Adam flow.

Consequently, qualitative properties established for the first-order model (e.g., Lyapunov stability) can be carried over to the accelerated dynamics by treating both the inertial contribution and the kernel-induced correction in the moments as perturbative terms. Crucially, the size of this perturbation is controlled by
\[
\rho \;:=\; \max\left\{\alpha,\ \Big(\tfrac{\alpha}{1-\beta_1}\Big)^2,\ \tfrac{\alpha}{1-\beta_2}\right\}.
\]
In particular, for fixed $\beta_1,\beta_2\in[0,1)$ and $\alpha\to 0$, one has $\rho=\mathcal{O}(\alpha)$, so the accelerated model can be regarded as an $\mathcal{O}(\alpha)$ refinement of the first-order dynamics. On the other hand, if $\beta_1$ is close to $1$ and $\alpha$ is not sufficiently small, the quadratic term $(\alpha/(1-\beta_1))^2$ may dominate and the discrepancy is then governed by this $\beta_1$-dependent contribution. Similarly, if $\beta_2$ is close to $1$ and $\alpha$ is not sufficiently small, the term $\alpha/(1-\beta_2)$ may dominate. In these cases, the corresponding $\beta$-dependent term(s) should be kept explicitly rather than absorbed into an $\mathcal{O}(\alpha)$ remainder.

A further point is that, unlike the first-order flow, the accelerated formulation explicitly involves the second derivative \(\ddot{\theta}\). As a result, the uniform perturbative comparison on a fixed horizon requires an additional regularity input: Proposition~\ref{prop:alpha-recovery} assumes that for each \(T>0\), $0<\delta <T$ with $\alpha_0\in(0,\delta)$ there exists \(C_T>0\) such that
\[
\sup_{t\in[\delta-\alpha_0,T]} \|\ddot{\theta}(t)\| \le C_T,
\]
an assumption absent from the first-order theory. This is the main reason for introducing \(\delta>0\) and working on \(t\in[\delta,T]\): it excludes the early-time regime where an initial time may generate large accelerations and compromise such uniform bounds. The time shift \(s=t+\alpha\) is then accommodated by taking \(\alpha\) sufficiently small so that \(\delta>\alpha\), ensuring that \(t+\alpha\) remains within the natural domain of the shifted formulation.

With this convention, letting $\alpha\to 0$ with $\beta_1,\beta_2$ fixed, the higher–order remainders vanish and the time shift becomes negligible, so that the dynamics recovers exactly the first–order nonlocal Adam integro–differential equation studied in \cite{heredia2025modelingadagradrmspropadam}.

\section{Stability and convergence}\label{sec:stability}

The integro-differential representation derived in Section~\ref{Sec:Adam-dynamics} places Adam within a class of inertial nonlocal flows. In particular, Proposition~\ref{prop:alpha-recovery} shows that, on any fixed horizon $t\in[\delta,T]$ with $\delta>0$ and $\alpha<\delta$, the accelerated second-order dynamics is an $\mathcal{O}(\rho)$ perturbation of the first-order nonlocal Adam model, where
\[
\rho=\max\left\{\alpha,\ \Big(\tfrac{\alpha}{1-\beta_1}\Big)^2,\ \tfrac{\alpha}{1-\beta_2}\right\}.
\]

This perturbative structure makes it possible to transfer the Lyapunov methodology developed for the first-order nonlocal model to the inertial setting. More precisely, we work with the Lyapunov functional associated with the first-order dynamics in \cite{heredia2025modelingadagradrmspropadam} and show that, for $\alpha$ sufficiently small (so that the above perturbation parameter $\rho$ is small), the additional inertial contribution can be controlled as a vanishing perturbation. As a consequence, the second-order flow exhibits the same qualitative stability and convergence mechanisms as the first-order model, up to a neighborhood of size $\mathcal{O}(\rho)$ (which shrinks as $\rho\to 0$).

\paragraph{Assumptions.} Throughout this section, $f:\mathbb{R}^d\to\mathbb{R}$ is $C^4$ and bounded from below. We assume that the sublevel set
\[
S_0:=\{\theta\in\mathbb{R}^d:\ f(\theta)\le f(\theta(0))\}
\]
is compact. Moreover, we assume that the trajectory $\theta(t)$ of the (inertial) nonlocal dynamics remains in $S_0$ for all $t\ge 0$. On $S_0$, $\nabla f$ is $L$-Lipschitz and we set $g(\theta):=\nabla f(\theta)$.

Since $S_0$ is compact and $f$ is continuous, we define
\[
f_\star:=\min_{\theta\in S_0} f(\theta),
\qquad
\Phi(t):=f(\theta(t))-f_\star \;\ge\; 0.
\]
Whenever the limit $f_\infty:=\lim_{t\to\infty} f(\theta(t))$ exists, one may equivalently work with $\Phi(t)=f(\theta(t))-f_\infty$. In particular, if $\theta(t)\to\theta^*$ for some critical point $\nabla f(\theta^*)=0$, then $f_\infty=f(\theta^*)$. For notational simplicity, we drop component-wise indices as Proposition~\ref{prop:alpha-recovery} and denote by $\|\cdot\|$ the standard Euclidean norm on $\mathbb{R}^d$.

Finally, we additionally assume that:
\begin{itemize}
  \item $\theta\in C^3([0,T];\mathbb{R}^d)$, $g\circ\theta\in C^3([0,T];\mathbb{R}^d)$, and $\|g(\theta(\cdot))\|^2\in C^3([0,T];\mathbb{R})$;
  \item the second derivatives of $g(\theta(\cdot))$ and $\|g(\theta(\cdot))\|^2$ are uniformly bounded on $[0,T]$;
  \item $v^{(j)}(t)\ge 0$ for all $t\in[0,T]$ and $j\in\{1,2\}$;
\end{itemize}

We fix a horizon away from the initial time: let $T>0$ and $0<\delta<T$, and assume $\alpha$ is small enough so that $\alpha<\delta$. Therefore, we assume that there exists a constant $C_T>0$, independent of $\alpha$, such that $\sup_{t\in[\delta-\alpha,T]}\|\ddot{\theta}(t)\|\le C_T$. Moreover, there exist a constant $\varepsilon_0>0$ such that $\varepsilon(t)\ge \varepsilon_0$ for all $t\in[\delta,T]\,.$ 

Moreover, let $\lambda_1:=(1-\beta_1)/\alpha$ and define the first-order moment $M_1(t)$ by
\[
\dot M_1(t)+\lambda_1 M_1(t)=\lambda_1 g(\theta(t)),\qquad M_1(0)=0.
\]
Equivalently, $M_1(t)=m^{(1)}(t)$ is the first-order nonlocal convolution with kernel $H^{(1)}_1$ from Proposition~\ref{prop:alpha-recovery}, evaluated along the same trajectory $\theta$. Similarly, let $\lambda_2:=(1-\beta_2)/\alpha$ and define the first-order second moment $M_2(t)$ by
\[
\dot M_2(t)+\lambda_2 M_2(t)=\lambda_2 \|g(\theta(t))\|^2,\qquad M_2(0)=0,
\]
so that $M_2(t)=v^{(1)}(t)$ is the first-order nonlocal convolution associated with $H^{(1)}_2$. We then define the effective scalar
\[
\mu(t):=\frac{\eta(t)}{\sqrt{M_2(t)}+\varepsilon(t)}.
\]

Following \cite{JMLRBelotto}, we consider the Lyapunov functional
\begin{equation}\label{eq:Lyap-V}
V(t):=\Phi(t)+\frac{\mu(t)}{2\lambda_1}\,\|M_1(t)\|^2,
\end{equation}
which is well-defined on $[\delta,T]$ under the standing assumptions on $\mu$.

Finally, in order to derive convergence statements beyond Lyapunov stability estimates, we assume that $f$ satisfies either a Polyak-{\L}ojasiewicz (PL) inequality or, more generally, a Kurdyka-{\L}ojasiewicz (KL) property (on $S_0$). This framework encompasses a large class of objectives encountered in applications and is well aligned with Lyapunov-based analysis for nonlocal flows.

Let us clarify that we do not pursue the fully general setting of merely $L$-smooth objectives (i.e., $\nabla f$ Lipschitz) without additional geometry such as PL or KL. In that generality, our analysis controls the inertial nonlocal dynamics only through its $O(\rho)$ perturbative relation to the first-order nonlocal flow on horizons away from the initial time. This is enough to inherit Lyapunov-type stability information (e.g., integral dissipation), but---without an error-bound condition---it does not yield comparably strong convergence conclusions (e.g., convergence of $f(\theta)$, pointwise convergence of $\theta(t)$, or quantitative rates). By contrast, PL/KL geometry provides the error-bound structure that converts dissipation into decay of the objective residual, which is what enables the stronger convergence statements proved below.

\begin{theorem}\label{thm:robust-inertial-adam}
Assume the assumptions above and $\beta_1 \leq \sqrt{\beta_2}$, and let $\theta$ solve the inertial Adam dynamics of~\eqref{eq:inertial}. Fix $0<\delta<T$ and assume $0<\alpha<\delta$. Then there exist $\alpha_0>0$ and constants $c_0>0$, $C_0>0$, independent of $\alpha\in(0,\alpha_0]$, such that the following hold.

\begin{itemize}
\item \textbf{Basic dissipation and averaged moment control.}
For all $T-\alpha\ge \delta$,
\[
V(T-\alpha) -V(\delta) 
\le
-c_0\int_\delta^{T-\alpha} \|M_1(t)\|^2\,dt+C_0\,\rho\,(T-\alpha-\delta).
\]
In particular,
\[
\frac{1}{T-\alpha-\delta}\int_\delta^{T-\alpha} \|M_1(t)\|^2\,dt
\le \frac{V(\delta)}{c_0 (T-\alpha-\delta)} +  \frac{C_0}{c_0}\,\rho.
\]

\item \textbf{PL case (including strong convexity).}
Assume that $f$ satisfies a Polyak-{\L}ojasiewicz inequality on $S_0$: there exists $\mu_{\rm PL}>0$ such that
\[
\frac12\|g(\theta)\|^2\ge \mu_{\rm PL}\,(f(\theta)-f_\star),
\qquad \forall\theta\in S_0.
\]
Then there exist constants $\omega>0$ and $C>0$, independent of $\alpha$, such that for $t\in[\delta,T-\alpha]$,
\[
\Phi(t)\le V(\delta)\,e^{-\omega (t-\delta)}+C\,\rho^2.
\]
If moreover $f$ is $\mu$-strongly convex on $S_0$ with unique minimizer $\theta^\ast$, then
\[
\|\theta(t)-\theta^\ast\|\le C''\,e^{-\omega (t-\delta)/2} +  C'\,\rho, \qquad \text{with} \qquad t\in[\delta,T-\alpha] \quad \text{and} \quad C', C'' >0\,.
\]

\item \textbf{General KL case.}
Assume that $f$ satisfies a Kurdyka-{\L}ojasiewicz inequality on $S_0$ with exponent $\sigma\in(0,1)$: there exist $C_{\rm KL}>0$ and a neighborhood of the critical set such that
\[
\|g(\theta)\|\ge C_{\rm KL}\,(f(\theta)-f_\star)^\sigma.
\]
Then there exist a rate function $p_\sigma$ (the same profile as in the first-order nonlocal model \cite{heredia2025modelingadagradrmspropadam}) and $C>0$ independent of $\alpha$ such that for $t\in[\delta,T-\alpha]$,
\[
\Phi(t)\le p_\sigma(t-\delta)+C\,\rho^{1/(2\sigma)}.
\]
\end{itemize}
\end{theorem}

\begin{proof}
\textbf{Basic dissipation — proof of item (1).}
By Proposition~\ref{prop:alpha-recovery}, for any fixed $0<\delta<T$ the second-order moments $(m^{(2)},v^{(2)})$ recover their first-order counterparts $(m^{(1)},v^{(1)})=(M_1,M_2)$ uniformly on $[\delta,T]$, with
\[
\|m^{(2)}(t)-M_1(t)\|\le C\Big(\frac{\alpha}{1-\beta_1}\Big)^2,
\qquad
|v^{(2)}(t)-M_2(t)|\le C\Big(\frac{\alpha}{1-\beta_2}\Big)^2,
\qquad t\in[\delta,T].
\]
Moreover, the normalized force term satisfies, uniformly on $[\delta,T]$,
\[
\frac{m^{(2)}(t)}{\sqrt{v^{(2)}(t)}+\varepsilon(t)}
=
\frac{M_1(t)}{\sqrt{M_2(t)}+\varepsilon(t)} + \mathcal{O}\!\left(
\max\Big\{\Big(\tfrac{\alpha}{1-\beta_1}\Big)^2,\;\tfrac{\alpha}{1-\beta_2}\Big\}
\right).
\]
Rewriting the inertial shifted equation and using $\frac{\alpha}{2}\ddot\theta(t)=\mathcal{O}(\alpha)$ on $[\delta-\alpha,T]$, we obtain the perturbed first-order form (away from the initial time)
\[
\dot\theta(t)
=
-\,\mu(t+\alpha)\,M_1(t+\alpha) + R_\rho(t)
=: F(t+\alpha)+R_\rho(t),
\qquad t\in[\delta-\alpha,\,T-\alpha],
\]
where
\[
\mu(t):=\frac{\eta(t)}{\sqrt{M_2(t)}+\varepsilon(t)},
\qquad
\|R_\rho(t)\|\le C_R\,\rho,
\qquad
\rho:=\max\Big\{\alpha,\;\Big(\tfrac{\alpha}{1-\beta_1}\Big)^2,\;\tfrac{\alpha}{1-\beta_2}\Big\},
\]
for some constant $C_R>0$ independent of $\alpha$. By construction, the drift term $F(t)$ coincides with the right-hand side of the first-order nonlocal dynamics analyzed in
\cite{heredia2025modelingadagradrmspropadam}.

We construct the Lyapunov function
\[
V(t) := \Phi(t) + \frac{\mu(t)}{2\lambda_1}\,\|M_1(t)\|^2,
\]
where $g(t):=\nabla f(\theta(t))$, $\Phi(t) := f(\theta(t)) - f_\star$, and $M_1$ solves
\[
\dot M_1(t) + \lambda_1 M_1(t) = \lambda_1 g(t),
\qquad M_1(0)=0.
\]

Differentiating $\Phi$ along a solution $\theta(t)$ of the inertial dynamics, we obtain
\[
\dot\Phi(t)
= \langle g(t),\dot\theta(t)\rangle
= \langle g(t),F(t)\rangle + \langle g(t),\widetilde R_\rho(t)\rangle\,,
\]
where
\[
\widetilde R_\rho(t)
:= R_\rho(t) + \bigl[F(t+\alpha)-F(t)\bigr].
\]
As shown in Appendix~\ref{app:shift-estimate}, for any fixed $0<\delta<T$ (and $\alpha<\delta$) the
shift estimate yields $\|F(t+\alpha)-F(t)\|\le C_{\delta,T}\alpha$ uniformly for $t\in[\delta,\,T-\alpha]$ with $\beta_1\leq \sqrt{\beta_2}$. Consequently, the redefined residual satisfies
\[
\|\widetilde R_\rho(t)\|\le C_{\widetilde R,\delta,T}\,\rho,
\qquad t\in[\delta,\,T-\alpha],
\]
for a constant $C_{\widetilde R,\delta,T}>0$ independent of $\alpha$.

On the other hand, the evolution equation for $M_1$ does not involve the perturbation $\widetilde R_\rho$, so the computation of the time derivative of the second term in $V(t)$ is exactly the same as in the first–order case (Theorem~3 in \cite{heredia2025modelingadagradrmspropadam}). Indeed,
\[
\frac{d}{dt}\!\left(\frac{\mu}{2\lambda_1}\,\|M_1\|^2\right)
=
\frac{\dot\mu}{2\lambda_1}\,\|M_1\|^2
+
\mu\,\langle M_1,g\rangle
-
\mu\,\|M_1\|^2.
\]
Adding this contribution to $\dot\Phi(t)$ we obtain
\[
\dot V(t)
= \dot V_1(t) +\langle g(t),\widetilde R_\rho(t)\rangle.
\]
where $\dot V_1(t)$ is:
\[
\dot V_1(t):=\mu(t) \left[\frac{h(t)}{2\lambda_1}-1 \right]\,\|M_1(t)\|^2\,,
\]
and $h(t):=\frac{d}{dt} \log \mu(t)$. The quantity $\dot V_1(t)$ is precisely the time derivative of the Lyapunov functional associated with the first–order nonlocal model. As shown in Corollary~1 of \cite{heredia2025modelingadagradrmspropadam}, one has the estimate $h(t) \leq C_h$ for all $t>0$ and, consequently,
\[
\dot V_1(t)
\le - \mu(t)\left[ 1- \frac{C_h}{2\lambda_1}\right]\,\|M_1(t)\|^2,
\]
for some $C_h>0$. Moreover, as shown in Theorem~3 of \cite{heredia2025modelingadagradrmspropadam},  $\mu(t)\geq \mu_0 >0$, and then
\[
\dot V(t)
\le
-\,c_0\,\|M_1(t)\|^2
+
\langle g(t),\widetilde R_\rho(t)\rangle
\]
for a $c_0$ positive as long as $C_h/2\lambda_1 <1$.
Using $\|\widetilde R_\rho(t)\|\le C_{\widetilde R, \delta,T} \rho$ and Cauchy–Schwarz,
\[
|\langle g(t),\widetilde R_\rho(t)\rangle|
\le C_{\widetilde R,\delta,T}\,\rho\,\|g(t)\|.
\]
On the sublevel $S_0$, the Lipschitz continuity of $\nabla f$ and the lower boundedness of $f$---Lemma~5(a) in \cite{heredia2025modelingadagradrmspropadam}---imply
\[
\|g(t)\|^2
\le C_g\,\Phi(t) \le C_g\,V(t),
\]
so that
\[
|\langle g(t),\widetilde R_\rho(t)\rangle|
\le C'_{\widetilde R,\delta,T}\,\rho\,V(t)^{1/2},
\]
for some constant $C'_{\widetilde R,\delta,T}>0$. In summary, we have obtained the key inequality
\begin{equation}\label{eq:key-V-ineq}
\dot V(t)
\le
-\,c_0\,\|M_1(t)\|^2
+
C'_{\widetilde R,\delta,T}\,\rho\,V(t)^{1/2}.
\end{equation}

Let us find a uniform bound. Since $S_0$ is compact and $g=\nabla f$ is continuous, $g$ attains its maximum on $S_0$. Hence there exists
\[
G_{\max}:=\max_{\theta\in S_0}\|g(\theta)\|<\infty.
\]
Therefore, whenever $\theta(t)\in S_0$ we have $\|g(t)\|=\|g(\theta(t))\|\le G_{\max}$. Using Cauchy-Schwarz and the perturbation bound $\|\widetilde R_\rho(t)\|\le C_{\widetilde R,\delta,T}\,\rho$, it follows that
\[
|\langle g(t),\widetilde R_\rho(t)\rangle|
\le \|g(t)\|\,\|\widetilde R_\rho(t)\|
\le G_{\max}\,C_{\widetilde R,\delta,T}\,\rho
=: C''_{\widetilde R,\delta,T}\,\rho,
\]
where $C''_{\widetilde R,\delta,T}:=G_{\max}C_{\widetilde R,\delta,T}$. Therefore, \eqref{eq:key-V-ineq} becomes
\begin{equation}\label{eq:key-V-ub}
\dot V(t)
\le
-\,c_0\,\|M_1(t)\|^2
+
C''_{\widetilde R,\delta,T}\,\rho\,.
\end{equation}
Integrating \eqref{eq:key-V-ub} over $[\delta,\,T-\alpha]$ we obtain
\[
V(T-\alpha) - V(\delta)
\le
-\,c_0 \int_\delta^{T-\alpha} \|M_1(t)\|^2\,dt
+
C''_{\widetilde R,\delta,T}\,\rho\,(T - \alpha-\delta).
\]
Since $V\ge 0$, we can rearrange to get
\[
\int_\delta^{T-\alpha} \|M_1(t)\|^2\,dt
\le
\frac{V(\delta)}{c_0}
+
\frac{C''_{\widetilde R,\delta,T}}{c_0}\,\rho\,(T-\alpha -\delta).
\]
Therefore, for every fixed $T>\delta+\alpha$ we obtain the finite-horizon averaged bound
\[
\frac{1}{T-\alpha-\delta}\int_\delta^{T-\alpha} \|M_1(t)\|^2\,dt 
\;\le\; \frac{V(\delta)}{c_0(T-\alpha-\delta)} + \frac{C''_{\widetilde R,\delta,T}}{c_0}\,\rho,
\]
which shows that the dissipation estimate of the first-order model is preserved up to a residual $\mathcal{O}(\rho)$ term on $[\delta,\,T-\alpha]$.

\medskip
\textbf{PL case — proof of item (2).}

Assume now that $f$ satisfies the Polyak-Łojasiewicz (PL) inequality \cite{PL} on $S_0$. Whenever $\theta(t)\in S_0$ for all $t\ge 0$ implies
\[
\Phi(t)\;\le\; \frac{1}{2\mu_{\mathrm{PL}}}\,\|g(t)\|^2.
\]
As recalled above, in the absence of the perturbation $\widetilde R_\rho$ the Lyapunov functional $V_1$ associated with the first–order nonlocal dynamics satisfies $\dot V_1(t) \le - \gamma\,V_1(t)$ for some $\gamma>0$, and there exists a constant $C>0$ such that $
\|M_1(t)\|^2 \le C\,V(t)\,.$ In the presence of the perturbation $\widetilde R_\rho$, the same computations as before lead to
\[
\dot V(t)
\le -\gamma\,V(t) + C'_{\widetilde R,\delta,T}\,\rho\,V(t)^{1/2}, \qquad t\in[\delta,T-\alpha]\,,
\]
for some constant $C'_{\widetilde R,\delta,T}>0$ independent of $\rho$. Using the elementary inequality $ab\le \epsilon a^2 + b^2/(4\epsilon)$ with $a=V^{1/2}$ and $b=C'_{\widetilde R,\delta,T}\rho$, and choosing $\epsilon>0$ small enough, we can absorb the perturbative term $C'_{\widetilde R,\delta,T}\,\rho\,V^{1/2}$ into $-\gamma\,V$. This yields for $t\in[\delta,T-\alpha]$
\[
\dot V(t)
\le -(\gamma-\varepsilon)\,V(t) + \frac{C_{\widetilde R,\delta,T}^{'2} \rho^2}{4\varepsilon},
\]
for any $\varepsilon\in(0,\gamma)$. Moreover, if we fix, for instance, $\varepsilon=\gamma/2$, this simplifies to
\begin{equation}\label{eq:v-ineq}
 \dot V(t)
\;\le\;
-\frac{\gamma}{2}\,V(t) + C'''_{\widetilde R,\delta,T}\,\rho^2,   
\end{equation}
for some constant $C'''_{\widetilde R,\delta,T}>0$ independent of $\rho$. 

To make explicit the competition between the dissipative term and the perturbation, we introduce an $\rho$–dependent threshold $\delta_V(\rho)$ by equating:
\[
\frac{\gamma}{2}\,\delta_V(\rho) = 2\,C'''_{\widetilde R,\delta,T}\,\rho^2
\qquad\Longrightarrow\qquad
\delta_V(\rho) = \frac{4}{\gamma}C'''_{\widetilde R,\delta,T}\,\rho^2.
\]
Then, whenever $V(t)\ge\delta_V(\rho)$, we have
\[
\frac{\gamma}{2}\,V(t)
\;\ge\;
\frac{\gamma}{2}\,\delta_V(\rho)
= 2\,C'''_{\widetilde R,\delta,T}\,\rho^2,
\]
and the previous inequality yields
\[
\dot V(t)
\;\le\;
-\frac{\gamma}{2}\,V(t) + C'''_{\widetilde R,\delta,T}\,\rho^2
\;\le\;
-\frac{\gamma}{2}\,V(t) + \frac{\gamma}{4}\,V(t)
= -\frac{\gamma}{4}\,V(t).
\]
In other words, as long as $V(t)\ge\delta_V(\rho)$, one has
\[
\dot V(t)\;\le\;-\frac{\gamma}{4}\,V(t),
\]
so $V(t)$ decays exponentially fast in this region. Once $V(t)$ enters the strip $\{V\le\delta_V(\rho)\}$, the inequality \ref{eq:v-ineq} shows that $V(t)$ remains bounded by a plateau of order $O(\rho^2)$. Since $\Phi(t)\le V(t)$, this implies a plateau of order $O(\rho^2)$ for the gap:
\[
\Phi(t)\;\leq V(\delta) e^{-\frac{\gamma}{2}(t-\delta)} + \frac{2}{\gamma}C'''_{\widetilde R,\delta,T}\rho^2
\]
Moreover, if $f$ is $\mu$–strongly convex on $S_0$ with unique minimizer $\theta^\ast$, then $\Phi(t) \;\ge\; \frac{\mu}{2}\,\|\theta(t)-\theta^\ast\|^2$ implies
\[
\|\theta(t)-\theta^\ast\| \;\leq\; C^{\prime\prime} e^{-\frac{\gamma}{4}(t-\delta)} + C^\prime\rho, \qquad \text{for} \qquad t\in[\delta, T-\alpha]\,,
\]
where we have used the elementary inequality $\sqrt{a+b} \leq \sqrt{a} + \sqrt{b}\,,$ so the distance to $\theta^\ast$ decays exponentially fast down to an $O(\rho)$-neighborhood.

\medskip
\textbf{KL case — proof of item (3).}

Finally, suppose that $f$ satisfies the Kurdyka-Łojasiewicz inequality on $S_0$, and the trajectory $\theta(t)$ remains in $S_0$ for all $t\ge0$.  As shown above, the dynamics can be written in the perturbed form $\dot\theta(t) \;=\; F(t) + \widetilde R_\rho(t)$. Therefore, for $t\in[\delta,T-\alpha]$
\[
\dot\Phi(t) = \langle g(t),F(t)\rangle + \langle g(t),\widetilde R_\rho(t)\rangle.
\]
For the first term, we have 
\[
\langle g(t),F(t)\rangle  \;=\; -\mu(t)\langle g(t), M_1(t) \rangle  \;\le\; -\frac{1}{2}\mu (t)\,\|g(t)\|^2 + \frac{1}{2}\mu(t) \| M_1-g \|^2\,,
\]
where we have applied the Cauchy-Schwarz and Young's inequalities for last step. In Appendix \ref{app:tracking-error} is shown that, for any fixed $\delta>0$, there exists
$C_{\delta,T}>0$ independent of $\alpha$ such that, for all $t\in[\delta,T-\alpha]$ and all $\alpha$ small enough,
\begin{equation}\label{eq:M1-tracking-alpha}
\|M_1(t)-g(t)\|\le C_{\delta,T}\,\frac{\alpha}{1-\beta_1}.
\end{equation}
Therefore, using also the uniform bound $\mu(t)\le \mu_{\max}$ on $[\delta,T-\alpha]$, we obtain
\begin{equation}\label{eq:M1-tracking-alpha2}
\frac{1}{2}\mu(t)\|M_1(t)-g(t)\|^2
\le \frac{1}{2}\mu_{\max} C_{\delta,T}^2\,\frac{\alpha^2}{(1-\beta_1)^2} \leq C^{IV}_{\delta,T} \rho,
\qquad t\in[\delta,T-\alpha].
\end{equation}
for $C^{IV}_{\delta,T}>0$ independent of $\rho$; therefore, this term can be absorbed into $|\langle g(t),\widetilde R_\rho\rangle|$ since, whenever $\theta(t)\in S_0$ for all $t\ge0$, we have $\|g(t)\|\le G_{\max}$ on the trajectory, and $|\langle g(t),\widetilde R_\rho\rangle| \leq C''_{\widetilde R,\delta,T}\,\rho$ for $t\in[\delta,T-\alpha]$. Therefore, using the KL inequality $\|g(t)\|\ge C_{\mathrm{KL}}\Phi(t)^{\sigma}$, $\mu(t) \geq \mu_0$, and combining these bounds, we obtain
\[
\dot\Phi(t)
\;\le\;
-\frac{\mu_0}{2} C_{\mathrm{KL}}^2\,\Phi(t)^{2\sigma}
+ C_{2,T}\,\rho
= -k_1\,\Phi(t)^{2\sigma} + k_2\,\rho,
\]
for $t\in[\delta,T-\alpha]$ and some constants $C_{2,T}, k_1,k_2>0$ independent of $\rho$. 

We now compare the dissipative term $k_1\Phi^{2\sigma}$ with the perturbation $k_2\rho$. Define
\[
\Lambda(\rho)
:= \left(\frac{2k_2}{k_1}\right)^{1/(2\sigma)}\rho^{1/(2\sigma)}.
\]
If $\Phi(t)\ge \Lambda(\rho)$, then $k_1\Phi(t)^{2\sigma}\ge k_1\Lambda(\rho)^{2\sigma}=2k_2\rho$, and the previous inequality gives
\[
\dot\Phi(t)
\;\le\;
-k_1\,\Phi(t)^{2\sigma} + k_2\rho
\;\le\;
-\tfrac12 k_1\,\Phi(t)^{2\sigma}\,, \qquad \text{with} \qquad t\in[\delta,T-\alpha]\,.
\]
Thus, as long as $\Phi(t)\ge\Lambda(\rho)$, the perturbed dynamics satisfies the same differential inequality as the first–order nonlocal model, up to the harmless factor $1/2$ in front of $k_1$ \cite{heredia2025modelingadagradrmspropadam}. The Kurdyka-Łojasiewicz argument used in \cite{heredia2025modelingadagradrmspropadam} (stratification of the trajectory into sublevel sets and integration of the resulting scalar inequality) therefore yields the same decay profile $p_\sigma(t)$ determined by the exponent $\sigma$: there exists a rate function $p_\sigma:[0,\infty)\to[0,\infty)$, depending only on $\sigma$, such that for every $T-\alpha>\delta$ and for all $t\in[\delta,T-\alpha]$ satisfying $\Phi(t)\ge\Lambda(\rho)$,
\[
\Phi(t)\le p_\sigma(t-\delta).
\]
Once $\Phi(t)$ enters the strip $\{\Phi\le\Lambda(\rho)\}$ at time $t_\rho$, the previous construction shows that
\[
\Phi(t)\;\le\;\Lambda(\rho)
= C\,\rho^{1/(2\sigma)}, \qquad t\ge t_\rho,
\]
for some $C>0$ independent of $\alpha$. In other words, for each fixed $\alpha>0$ small enough the trajectory $\theta(t)$ is attracted, with the same Kurdyka-Łojasiewicz rate $p_\sigma(t)$ as in the first–order nonlocal model, to an $O(\rho^{1/(2\sigma)})$-neighbourhood in function values of the limiting critical level. In the vanishing–stepsize regime $\alpha\to 0$ one has $\Lambda(\rho)\to 0$ and the perturbation term disappears, so that one recovers the exact KL convergence results of \cite{heredia2025modelingadagradrmspropadam}.
\end{proof}

We conclude this section by summarizing the main implications of Theorem~\ref{thm:robust-inertial-adam}:

Item~(1) yields an integral Lyapunov dissipation estimate: for a fixed stepsize $\alpha>0$, the functional $V$ may not be strictly decreasing, yet its time-averaged dissipation over any window $[\delta,\,T-\alpha]$ remains controlled up to an $\mathcal{O}(\rho)$ supply term. In turn, since this $\mathcal{O}(\rho)$ contribution does not generally vanish when $\alpha$ is held constant, one should not expect convergence to an exact critical point in full generality; rather, the appropriate conclusion is practical convergence to a $\rho$-dependent neighbourhood.

Items~(2) and~(3) make this heuristic precise under additional geometric structure. In the PL case one recovers exponential decay of the objective residual, up to a plateau of order $O(\rho^2)$; under strong convexity one additionally obtains convergence in parameter space up to an $O(\rho)$ neighbourhood of the minimizer. In the general KL setting one obtains the same rate profile $p_\sigma$ as for the first-order nonlocal model, up to a plateau of order $O(\rho^{1/(2\sigma)})$ determined by the KL exponent $\sigma$. In both regimes, letting $\rho\to 0$ (in particular, $\alpha\to 0$ for fixed $(\beta_1,\beta_2)$) removes the perturbation and recovers the corresponding first-order behaviour.

A further interpretation of item~(1) is that it provides an averaged near-stationarity property. Indeed,
\begin{corollary}
Fix $\delta>0$ and $T>\delta+\alpha$. Under the assumptions of Theorem~\ref{thm:robust-inertial-adam}, there exists $t^*\in[\delta,\,T-\alpha]$ such that
\begin{equation}\label{eq:cor-near-stationary-time}
\|M_1(t^*)\|^2 \;\le\; \frac{1}{T-\alpha-\delta}\int_{\delta}^{T-\alpha}\|M_1(t)\|^2\,dt
\;\le\; \frac{V(\delta)}{c_0\,(T-\alpha-\delta)}+\frac{C_0}{c_0}\,\rho.
\end{equation}
Equivalently,
\begin{equation}\label{eq:cor-near-stationary-time-sqrt}
\|M_1(t^*)\|\;\le\;\left(\frac{V(\delta)}{c_0\,(T-\alpha-\delta)}+\frac{C_0}{c_0}\,\rho\right)^{1/2}.
\end{equation}
\end{corollary}

\begin{proof}
Since $t\mapsto \|M_1(t)\|^2$ is nonnegative on $[\delta,\,T-\alpha]$, we have
\[
\min_{t\in[\delta,\,T-\alpha]}\|M_1(t)\|^2 \;\le\; \frac{1}{T-\alpha-\delta}\int_{\delta}^{T-\alpha}\|M_1(t)\|^2\,dt.
\]
Choose $t^*\in[\delta,\,T-\alpha]$ attaining the minimum. Combining the above inequality with the integrated estimate in Theorem~\ref{thm:robust-inertial-adam} yields \eqref{eq:cor-near-stationary-time}, and \eqref{eq:cor-near-stationary-time-sqrt} follows by taking square roots.
\end{proof}

Since $M_1(t)$ is the exponentially weighted moving average of past gradients, this estimate shows that, for small stepsizes (hence small $\rho$) and over long time horizons, the time-averaged squared norm of $M_1(t)$ is small (of order $O(\rho)$, up to a transient $O(1/T)$). Nevertheless, a persistent $O(\rho)$ forcing may prevent exact convergence when the stepsize is kept constant. It is therefore natural to expect that adopting a decaying stepsize schedule $\alpha(t)\to 0$ would progressively remove this forcing and restore exact convergence, but we do not pursue this direction here.

\section{Nonlocal Lagrangian Formulation}\label{sec:Lagrangian}
In Section~\ref{Sec:Adam-dynamics}, we have derived a continuous-time formulation of the Adam optimization algorithm, in the form of an inertial integro–differential equation driven by exponentially weighted gradient moments (cf. Proposition~\ref{prop:Adam2}). We now adopt a more structural viewpoint and ask whether this dynamics can be embedded into a nonlocal variational framework.

To this end, rather than working directly with a fixed closed-form expression, we introduce an abstract nonlocal force field $F^i$ capturing the qualitative architecture of Adam---gradient information filtered through temporal kernels and adaptive rescaling---and we construct a nonlocal Adam-type Lagrangian whose Euler-Lagrange equations generate inertial flows of the form:
\[
\frac{\alpha}{2}\ddot\theta^i(t)+\dot\theta^i(t)+a(t)\,F^i(\theta,t)=0.
\]
For readers familiar with Lagrangian and Hamiltonian formalisms, it may already be apparent that the strictly causal nature of the Adam kernels is expected to be incompatible with an exact nonlocal Lagrangian description in the usual sense; to the best of the author's knowledge, a fully Lagrangian formulation of strict Adam is therefore (a priori) not available, at least as we are going to propose it. The variational framework developed below should thus be understood as defining an ideal, Adam-inspired, nonlocal Lagrangian class, which is conceptually linked to, but not identical with, the continuous-time Adam dynamics. This perspective is interesting both conceptually, as it connects Adam-like learning dynamics with nonlocal field-theoretic models, and algorithmically, as it provides a principled template for designing new optimization schemes inspired by Adam. The precise relation between this Adam-type Lagrangian class and the concrete continuous-time Adam dynamics will be clarified below.

Throughout this section, we use index notation and the Einstein summation convention: repeated indices (one upper and one lower) are implicitly summed over $i=1,\dots,n$.  We endow $\mathbb{R}^n$ with the Euclidean metric $\delta_{ij}$ (with inverse $\delta^{ij}$), so we freely raise/lower indices via $x_i=\delta_{ij}x^j$ and $\partial^i f=\delta^{ij}\partial_j f$. In particular, inner products and norms are written as $x\cdot y=\delta_{ij}x^i y^j$ and $\|x\|^2=\delta_{ij}x^i x^j$.

\subsection{Nonlocal theory in a nutshell}
We briefly recall the functional variational framework for time-nonlocal Lagrangians, following \cite{Heredia2021, heredia2025nonlocalmechanics}.

Let $\theta\in \mathcal{C}^\infty(\mathbb{R},\mathbb{R}^n)$ be a smooth trajectory. A convenient way to encode temporal nonlocality is to use time translations acting on trajectories. More precisely, we define the translation operator $T_t$ by
\begin{equation}\label{eq:translation-operator}
(T_t\theta)(\tau):=\theta(t+\tau),\qquad t,\tau\in\mathbb{R}.
\end{equation}
From this perspective, the map $t\mapsto T_t\theta$ should be viewed as an evolution of the entire curve within the space of trajectories; see Figure~\ref{fig:translation-operator}. This viewpoint is natural in time-nonlocal models, where the ``state at time $t$'' is not determined solely by the instantaneous value $\theta(t)$ (nor by finitely many of its derivatives), but rather by the full temporal profile $\tau\mapsto \theta(t+\tau)$.

\begin{figure}[h!]
\centering
\includegraphics[width=0.6\linewidth]{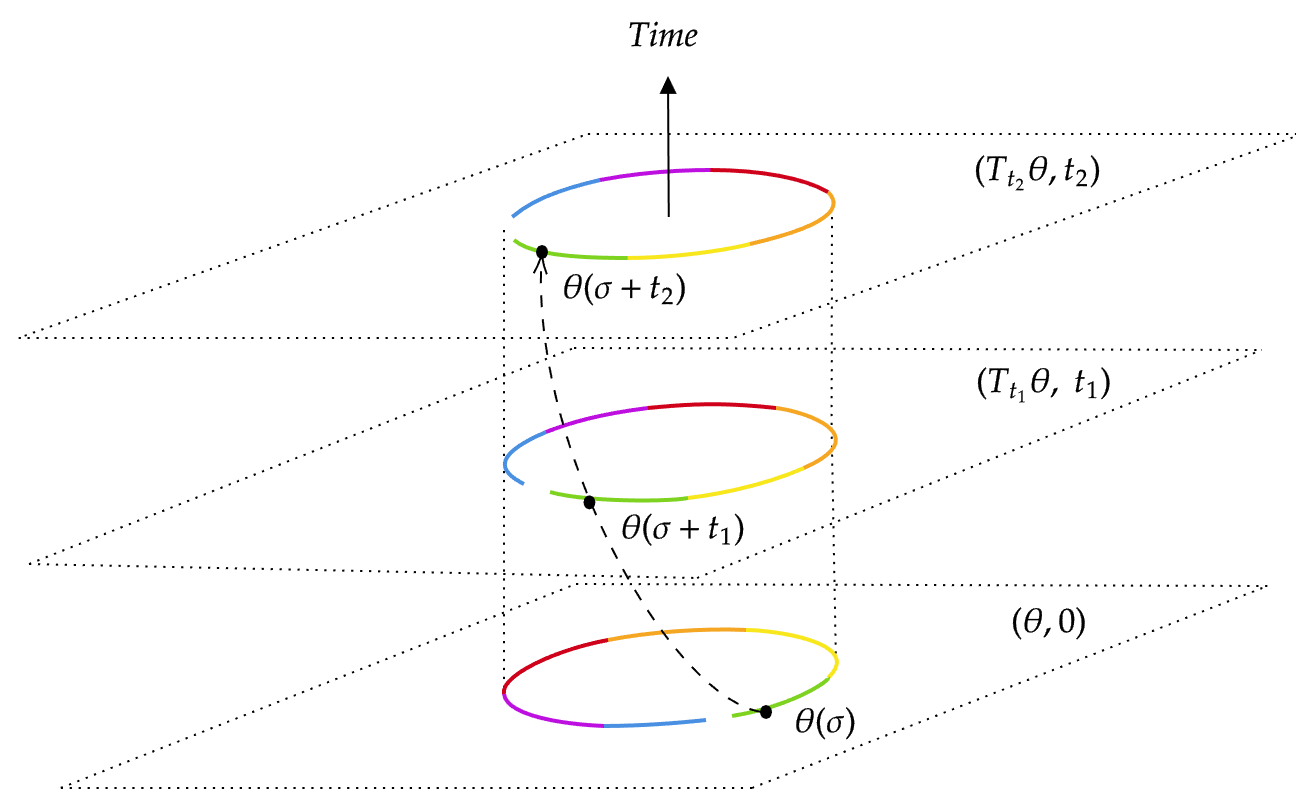}
\caption{\textbf{Geometric meaning of the translation operator.}
For a trajectory $\theta:\mathbb{R}\to\mathbb{R}^n$, the map $t\mapsto T_t\theta$ defines an orbit in the space of trajectories, and time-nonlocal models may depend on the full temporal profile $\tau\mapsto \theta(t+\tau)$ rather than only on the local value $\theta(t)$.}
\label{fig:translation-operator}
\end{figure}

This notation is also natural for describing memory terms defined by integral operators. For instance, convolution-type expressions can be written in terms of translated trajectories as
\[
(K\ast \theta)(t)=\int_{\mathbb{R}}K(\tau)\,\theta(t-\tau)\,d\tau
=\int_{\mathbb{R}}K(\tau)\,(T_t \theta)(-\tau)\,d\tau,
\]
so the value at time $t$ is obtained by applying the kernel to the entire translated path $T_t\theta$. This is exactly the mechanism underlying Adam: in continuous time, its first and second moments are exponentially weighted averages of past gradients, i.e.\ causal convolutions. Consequently, the resulting effective force can be viewed abstractly as a functional $F(T_t\theta,t)$ acting on the translated trajectory $T_t\theta$. Indeed, one has
\[
m^i(T_t\theta, t)\sim\int_{0}^{t}K_{\beta_1}(\tau)\,\partial^i f\!\bigl(T_t\theta(-\tau)\bigr)\,d\tau,\qquad 
v(T_t\theta, t)\sim\int_{0}^{t}K_{\beta_2}(\tau)\,\bigl\|\nabla f\!\bigl(T_t\theta(-\tau)\bigr)\bigr\|^2\,d\tau\,,
\]
so that both $m(T_t\theta,t)$ and $v(T_t\theta,t)$ depend on the past segment of the translated trajectory through causal kernels $K_{\beta_1},K_{\beta_2}$.

We now describe the nonlocal variational setup. Consider a dynamical system described by the nonlocal action integral
\begin{equation}\label{chap31-IA}
S([T_t\theta])=\int_{\mathbb{R}} dt\,L([T_t\theta],t),
\end{equation}
where the Lagrangian $L$ is allowed to depend on the whole translated path $T_t\theta$ (and possibly explicitly on $t$). From now on, to simplify notation, we will omit the functional brackets $[\cdot]$; however, it should be understood that $L(T_t\theta,t)$ is a functional of the full trajectory through $T_t\theta$.

\begin{definition}
The kinematic space is the function space $\mathcal{K}:=\mathcal{C}^\infty(\mathbb{R},\mathbb{R}^n)$. For time-dependent nonlocal Lagrangians, we introduce the concept of the extended kinematic space, denoted by \( \mathcal{K}' = \mathcal{K} \times \mathbb{R} \)
\end{definition}

\begin{definition}
A (time-dependent) nonlocal Lagrangian is a real-valued functional on $\mathcal{K}'$,
\[
L:\mathcal{K}'\to\mathbb{R},\qquad (\theta,t)\mapsto L(T_t\theta,t).
\]
\end{definition}

Since $L$ depends on the entire trajectory, the action over an unbounded time domain may fail to converge. We therefore introduce a one-parameter family of finite nonlocal action integrals, defined as:
\begin{equation}\label{chap31-SR}
S(\theta;R)=\int_{|t|\le R}dt\,L(T_t\theta,t),\qquad R>0,
\end{equation}
which is well-defined for every fixed $R$. 

\begin{proposition}
The nonlocal Euler-Lagrange equations associated with the nonlocal actions \eqref{chap31-SR} are
\begin{equation}\label{chap31-EOMS}
\psi_i(\theta,\sigma)=0,
\qquad
\psi_i(\theta,\sigma):=\int_{\mathbb{R}}dt\,\lambda_i(\theta,t,\sigma),
\end{equation}
where
\begin{equation}\label{chap31-Lambda}
\lambda_i(\theta,t,\sigma):=\frac{\delta L(T_t\theta,t)}{\delta\theta^i(\sigma)}.
\end{equation}
Here $\delta/\delta\theta^i$ denotes the functional derivative with respect to the $i$-th component of $\theta$.
\end{proposition}

\begin{proof}
The proof is in \cite{Heredia2021, heredia2025nonlocalmechanics}.
\end{proof}

Up to this point, the dynamics in the extended kinematic space $\mathcal{K}'$ have been described in terms of trajectories of the form $(T_t\theta,t)\in\mathcal{K}'$ starting from $(\theta,0)$. For time-explicitly dependent nonlocal Lagrangians it is often more convenient to express the equations in a form that is compatible with an arbitrary choice of reference time. This is achieved as follows:

\begin{proposition}
Define
\begin{equation}\label{chap-EOMmc}
\Psi_i(\theta,t,\sigma):=\psi_i(\theta,t+\sigma).
\end{equation}
Then the nonlocal Euler-Lagrange equations can be equivalently written as $\Psi_i(\theta,t,\sigma)=0$ for all $(t,\sigma)\in\mathbb{R}^2$.
In particular, for trajectories starting at $t=0$ one recovers $\psi_i(\theta,\sigma)=\Psi_i(\theta,0,\sigma)$.
\end{proposition}

\begin{proof}
The proof is in \cite{Heredia2021, heredia2025nonlocalmechanics}.
\end{proof}

\subsection{Bridge to Adam}\label{sec:Adam}
In this subsection we introduce an Adam-type nonlocal Lagrangian and show that its nonlocal Euler-Lagrange equations generate a broad family of inertial Adam-type flows. More precisely, we work with an abstract nonlocal force field $F^i(T_t\theta,t)$ whose architecture is inspired by the force terms $T^i(T_t\theta,t)$ appearing in Proposition~\ref{prop:Adam2}. Under a suitable weighted reciprocity assumption on the functional derivatives of $F^i$ (cf.~Proposition~\ref{prop:nonlocal-lagrangian}), we prove that the corresponding inertial dynamics
\[
\frac{\alpha}{2}\ddot\theta^i(t)+\dot\theta^i(t)+a(t)\,F^i(T_t\theta,t)=0
\]
is variational and can be obtained from an explicit nonlocal Lagrangian.

A key point is that this variational structure is not tied to a single closed-form realization of Adam; rather, it defines a whole variational Adam-type class parametrized by the choice of $F$ within the reciprocity hypothesis. Accordingly, we present:

\begin{proposition}\label{prop:nonlocal-lagrangian}
Let $\alpha>0$ denote the learning rate, let $a:\mathbb{R}\to\mathbb{R}$ be a given scalar weight function, and let $F_i(T_t\theta,t)$ denote a nonlocal force field. Assume the weighted reciprocity condition
\begin{equation}\label{eq:reciprocity}
W(\xi)\,\frac{\delta F_i(T_\xi\theta,\xi)}{\delta\theta^j(\sigma)}
=
W(\sigma)\,\frac{\delta F_j(T_\sigma\theta,\sigma)}{\delta\theta^i(\xi)},
\qquad
W(t):=e^{\frac{2t}{\alpha}}\,a(t),
\end{equation}
for all trajectories $\theta\in\mathcal{K}$ and all $\xi,\sigma\in\mathbb{R}$. Then the inertial Adam-type equation
\begin{equation}\label{eq:adamlike-inertial}
\frac{\alpha}{2} \ddot\theta^j(t)+\dot\theta^j(t)+a(t)\,F^j(T_t\theta,t)=0
\end{equation}
is generated by the nonlocal Lagrangian
\begin{equation}\label{eq:Lagrangian}
L(T_t\theta,t)
=
\frac{e^{\frac{2}{\alpha}t}}{4}
\left[
\alpha\,\dot \theta^2(t)
+ 2\,\theta^i(t)\dot\theta_i(t)
+ \frac{2}{\alpha}\theta^2(t)
- 4\,a(t)\,\theta^i(t)\int_0^1 F_i(\tau T_t\theta,t)\,d\tau
\right].
\end{equation}
\end{proposition}

\begin{proof}
As previously mentioned, the dynamical equations described by a time-explicitly
dependent nonlocal Lagrangian are given by \eqref{chap-EOMmc}. For this
demonstration, we split them into the local and nonlocal parts,
$\psi_j = \psi_{j,L} +\psi_{j,NL}$, to make the computation more transparent.

\noindent\textbf{Local part.}
We first consider the local contribution
\[
\psi_{j,L}(\theta, t+\sigma)
=
\int_\mathbb{R}\mathrm{d}\xi\;
\frac{\delta L_L(T_\xi \theta, \xi)}{\delta \theta^j(\sigma+t)}\,,
\]
where the local Lagrangian is
\[
L_L(T_\xi\theta,\xi)
=
\frac{e^{\frac{2\xi}{\alpha}}}{4}
\left[
\alpha\,\dot \theta^2(\xi)
+ 2\,\theta^i(\xi)\dot\theta_i(\xi)
+ \frac{2}{\alpha}\theta^2(\xi)
\right].
\]
Using the properties of the Dirac delta \cite{Vladimirov_GF}, one obtains
\[
\frac{\delta L_L(T_\xi \theta, \xi)}{\delta \theta^j(\sigma+t)}
=
\frac{e^{\frac{2\xi}{\alpha}}}{2}
\Big[
\big(\alpha\,\dot\theta_j(\xi) + \theta_j(\xi)\big)\,\dot\delta(\xi-(\sigma+t))
+
\big(\dot\theta_j(\xi) + \tfrac{2}{\alpha} \theta_j(\xi)\big)\,\delta(\xi-(\sigma+t))
\Big].
\]
Integrating over $\xi$ we get
\[
\psi_{j,L}(\theta,t+\sigma)
=
-e^{\frac{2(\sigma+t)}{\alpha}}
\left[
\frac{\alpha}{2}\ddot \theta_j(\sigma+t)
+ \dot \theta_j(\sigma+t)
\right].
\]

\noindent\textbf{Nonlocal part.}
For the nonlocal contribution we take
\[
L_{NL}(T_\xi \theta, \xi)
=- e^{\frac{2\xi}{\alpha}}\,
a(\xi)\,\theta^i(\xi)
\int^1_0 \mathrm{d}\tau\,F_i(\tau T_{\xi}\theta,\xi),
\]
so that $L_{NL}(T_t\theta,t)$ reproduces the nonlocal term in \eqref{eq:inertial}. A straightforward functional derivative yields
\begin{equation*}
\begin{split}
\int_\mathbb{R}\mathrm{d}\xi\,
\frac{\delta L_{NL}(T_\xi \theta, \xi)}{\delta \theta^j(\sigma+t)}
&=
-\int^1_0\mathrm{d}\tau\int_\mathbb{R}\mathrm{d}\xi\;
e^{\frac{2\xi}{\alpha}}\,a(\xi)
\\ &\qquad\times
\left[
F_j(\tau T_\xi\theta,\xi)\,\delta(\xi-(\sigma+t))
+
\tau\,\theta^i(\xi)\,
\frac{\delta F_i(\tau T_\xi\theta,\xi)}{\delta\theta^j(\sigma+t)}
\right].
\end{split}
\end{equation*}
Using the symmetry assumption (which holds for all trajectories and, in particular, for $\tau T_\xi\theta$) we can rewrite the
second term as\footnote{A more detailed description of this step is provided in Appendix~\ref{App:Prop5}.}
\begin{align*}
 \tau \int_{\mathbb{R}} \mathrm{d}\xi \,
 \theta^i(\xi)\,e^{\frac{2\xi}{\alpha}}\,a(\xi)\,
 \frac{\delta F_i(\tau T_\xi\theta, \xi)}{\delta \theta^j(\sigma + t)}
 &=
 \tau\,e^{\frac{2(\sigma + t)}{\alpha}} a(\sigma + t)
 \int_{\mathbb{R}} \mathrm{d}\xi \,
 \theta^i(\xi)\,
 \frac{\delta F_j(\tau T_{\sigma + t} \theta,\sigma+t)}{\delta \theta^i(\xi)} \\
 &= \tau\,e^{\frac{2(\sigma + t)}{\alpha}} a(\sigma + t)\,
 \frac{\mathrm{d}}{\mathrm{d}\tau}
 \Bigl[F_j(\tau T_{\sigma + t}\theta, \sigma + t)\Bigr],
\end{align*}
where in the last equality we used the functional chain rule along the curve
$\tau\mapsto \tau T_{\sigma+t}\theta$ in kinematic space. Therefore, the nonlocal part becomes
\begin{align*}
    \int_{\mathbb{R}} \mathrm{d}\xi\,
    \frac{\delta L_{\text{NL}}(T_\xi \theta, \xi)}{\delta \theta^j(\sigma + t)}
    &=
    -e^{\frac{2(\sigma + t)}{\alpha}} a(\sigma + t)
    \int_0^1 \mathrm{d}\tau\,
    \frac{\mathrm{d}}{\mathrm{d}\tau}\Bigl[\tau F_j(\tau T_{\sigma + t} \theta, \sigma + t)\Bigr] \\
    &=
    -e^{\frac{2(\sigma + t)}{\alpha}} a(\sigma + t)\,
    F_j(T_{\sigma + t} \theta, \sigma + t),
\end{align*}
since the contribution at $\tau=0$ vanishes.

\medskip
\noindent\textbf{Nonlocal EL equations.}
Summing the local and nonlocal contributions we obtain
\[
 \psi_j(\theta, \sigma + t)
 =
 -e^{\frac{2(\sigma + t)}{\alpha}}
 \left[
 \frac{\alpha}{2}\ddot \theta_j(\sigma+t)
 + \dot \theta_j(\sigma+t)
 +  a(\sigma + t)\,F_j(T_{\sigma + t} \theta, \sigma + t)
 \right] = 0.
\]
Since the exponential prefactor is never zero, this is equivalent to
\[
    \frac{\alpha}{2}\ddot \theta_j(\sigma+t)
    + \dot \theta_j(\sigma+t)
    +  a(\sigma + t)\,F_j(T_{\sigma + t} \theta, \sigma + t) = 0.
\]
Redefining $\sigma + t$ as $t$ we obtain precisely the inertial Adam-type equation~\ref{eq:adamlike-inertial}.
\end{proof}

Although the auxiliary parameter $\tau\in[0,1]$ appearing in the nonlocal Lagrangian of Proposition~\ref{prop:nonlocal-lagrangian} may at first sight seem to be introduced in an ad hoc manner, its origin has a clear geometric interpretation. In fact, $\tau$ should be understood as a homotopy parameter \cite{Hatcher2002AlgebraicTopology} in kinematic space, in direct analogy with the construction underlying the Poincar\'e lemma \cite{ChoquetBruhatDeWittMorette2004}. Recall that, a closed $1$-form $\Lambda=\Lambda_i(x)\,\mathrm{d}x^i$ admits the explicit potential
\begin{equation}\label{eq:poincare}
 f(x)=\int_0^1 x^i\,\Lambda_i(\tau x)\,\mathrm{d}\tau,   
\end{equation}
which is the $p=1$ case of the homotopy operator in the Poincar\'e lemma. Here $\tau\mapsto\tau x$ is the straight-line homotopy joining $0$ and $x$, and $\tau$ is only a path parameter. Differentiating $f$ yields
\begin{equation}\label{eq:diff_poincare}
 \partial_j f(x)
=
\int_0^1\Lambda_j(\tau x)\,\mathrm{d}\tau
+
\int_0^1\tau x^i\,\partial_j\Lambda_i(\tau x)\,\mathrm{d}\tau,   
\end{equation}
so a factor $\tau$ appears only because of the chain rule applied to $\Lambda_i(\tau x)$. Using the closedness condition $\partial_j\Lambda_i=\partial_i\Lambda_j$, the second term becomes a total $\tau$-derivative and the integral collapses to
\[
\partial_j f(x)=\Lambda_j(x).
\]
In other words, we can view $f$ as a potential for the $1$-form $\Lambda$: this is precisely the standard analogy
\[
\Lambda \ \text{closed} \ \Longrightarrow\ \exists\, f \ \text{such that } \mathrm{d}f=\Lambda.
\]

In our nonlocal setting, the object we want to integrate is the nonlocal $1$-form $F_j\!\bigl(T_{\xi}\theta,\xi\bigr)$ on kinematic space  and $L_{\mathrm{NL}}$ plays the role of its potential, in the same sense that $f$ is a potential for $\Lambda$ in the finite-dimensional Poincar\'e lemma. Accordingly, the straight-line homotopy $\tau\mapsto \tau x$ is replaced by the straight-line homotopy in kinematic space, $\tau \mapsto \tau\,T_\xi\theta\,,$ and the ansatz for $L_{\mathrm{NL}}$ is the functional analogue of~\eqref{eq:poincare}. Concretely, the factor $\theta^i(\xi)$ corresponds to $x^i$, and $F_i(\tau T_\xi\theta,\xi)$ corresponds to
$\Lambda_i(\tau x)\,.$

When we compute the functional derivative
\[
\int_\mathbb{R}\mathrm{d}\xi\,\frac{\delta L_{\mathrm{NL}}(T_\xi \theta, \xi)}{\delta\theta^j(\sigma+t)},
\]
the variation of the explicit factor $\theta^i(\xi)$ produces the term $\int_0^1 F_j(\tau T_\xi\theta,\xi)\,\mathrm{d}\tau$, while the variation of $F_i(\tau T_\xi\theta,\xi)$ through its dependence on the scaled trajectory $\tau T_\xi\theta$ produces a second contribution
\[
\int_0^1 \tau\,\theta^i(\xi)\,
\frac{\delta F_i(\tau T_\xi\theta,\xi)}{\delta\theta^j(\sigma+t)}\,\mathrm{d}\tau.
\]
This is the exact analogue of the $\tau x^i\partial_j\Lambda_i(\tau x)$ term of~\eqref{eq:diff_poincare}: the factor $\tau$ arises solely from the functional chain rule applied to the map $\theta\mapsto F_i(\tau T_\xi\theta,\xi)$ along the homotopy $\tau\mapsto\tau T_\xi\theta$.

Within this homotopy-based viewpoint, the role of the reciprocity condition becomes transparent. The nonlocal Lagrangian formulation provided by Proposition~\ref{prop:nonlocal-lagrangian} relies crucially on this assumption, which plays the role of a nonlocal closedness condition  $\partial_j\Lambda_i=\partial_i\Lambda_j$: it is the functional analogue of the symmetry of mixed partial derivatives in the finite-dimensional gradient case, or, equivalently, of $\mathrm{d}\Lambda=0$ for a $1$-form $\Lambda$. In particular, it must hold for all pairs of times $(\xi,\sigma)\in\mathbb{R}^2$. Under this reciprocity condition, the $\tau$–weighted term above can be rewritten as a total derivative,
\[
\int_\mathbb{R}\mathrm{d}\xi\,e^{\frac{2\xi}{\alpha}}\,a(\xi)\,\theta^i(\xi)\,
\frac{\delta F_i(\tau T_\xi\theta,\xi)}{\delta\theta^j(\sigma+t)}
=e^{\frac{2(\sigma + t)}{\alpha}} a(\sigma + t)\,\frac{\mathrm{d}}{\mathrm{d}\tau}
\bigl[F_j(\tau T_{\sigma+t}\theta,\sigma+t)\bigr],
\]
so that the $\tau$–integral reduces to $e^{\frac{2(\sigma + t)}{\alpha}}\,a(\sigma + t)\,F_j(T_t\theta,t)$. Thus the parameter $\tau$ that appears in the intermediate formulas is not an additional physical variable nor a genuine rescaling of the trajectory; it is merely the parameter of the homotopy in kinematic space.

On the other hand, the Adam-type forces introduced in Proposition~\ref{prop:Adam2} are built from causal convolutions of the form
\[
m_i(T_t\theta,t) \sim  \int_0^t K_{\beta_1}(t-\tau)\,\partial_i f(\theta(\tau))\,\mathrm{d}\tau,
\qquad
v(T_t\theta,t)
\sim \int_0^t K_{\beta_2}(t-\tau)\,\|\nabla f(\theta(\tau))\|^2\,\mathrm{d}\tau,
\]
with kernels $K_{\beta_a}$ supported on $[0,\infty)$ only. As a consequence, for each fixed time $t$ the quantity $F_i(T_t\theta,t)$ depends solely on the past history
$\{\theta(\tau):\tau\le t\}$, and one has
\[
\frac{\delta F_i(T_t\theta,t)}{\delta\theta^j(\sigma)} = 0
\quad\text{whenever}\quad \sigma>t,
\]
whereas, for a generic trajectory and objective function $f$, the functional derivative
\[
\frac{\delta F_j(T_\sigma\theta,\sigma)}{\delta\theta^i(t)} \neq 0 \quad\text{whenever}\quad \sigma>t\,.
\]
Hence the reciprocity condition necessarily fails off the diagonal $t=\sigma$ because of the built-in arrow of time in the Adam-type memory terms. On the diagonal, the temporal incompatibility disappears and both functional derivatives may be nonzero simultaneously. In that case, the reciprocity condition reduces to a local-in-time requirement, namely the symmetry of the instantaneous Jacobian
\[
\frac{\delta F_i(T_t\theta,t)}{\delta\theta^j(t)}
=
\frac{\delta F_j(T_t\theta,t)}{\delta\theta^i(t)}.
\]
This condition is the analogue of the pointwise symmetry of mixed partial derivatives in finite dimensions and is therefore necessary for a local potential at each fixed time $t$, but it is not sufficient to guarantee the existence of a global nonlocal potential in kinematic space. In particular, the Adam-type forces retain an essential dependence on the entire past history $\{\theta(\tau):\tau\le t\}$ through the memory kernels in $m$ and $v$, so that the functional derivatives off the diagonal $t\neq\sigma$ encode how perturbations at one time affect the force at another time. For the strictly causal convolutions underlying Adam, these off-diagonal derivatives have asymmetric support in time (they vanish whenever the perturbation lies in the future but not conversely), and thus the full reciprocity condition cannot hold. 

Hence, the continuous-time Adam dynamics should not be viewed as a direct instance of Proposition~\ref{prop:nonlocal-lagrangian}. Rather, Proposition~\ref{prop:nonlocal-lagrangian} characterizes an ideal, time-symmetric, variational Adam-type class. In fact, at the level of the integro-differential equation of motion one may formally match Adam to this template by identifying the nonlocal forcing via
\[
a(t)\,F(T_t\theta,t)\;\equiv\;\eta(t+\alpha)\,T(T_{t+\alpha}\theta,t+\alpha),
\]
thereby recovering an Adam-like evolution within the same structural ansatz.

This variational, time-symmetric core can be read as a principled design blueprint: new Adam-type methods may be constructed by starting from the reciprocity-preserving template and then introducing controlled, explicitly causal symmetry-breaking to recover the desired one-sided memory structure.

Beyond optimization, this perspective suggests that nonlocal variational tools developed in physics may become available for deep learning dynamics. In particular, one may leverage Noether-type arguments to derive conserved or approximately conserved quantities from symmetries of the nonlocal Lagrangian---for instance, scale invariance or other rescaling symmetries relevant to modern networks. Now, a further interesting direction is to investigate gauge-like symmetries and the corresponding constraints furnished by the second Noether theorem, which could impose nontrivial structural restrictions. In this sense, the present framework opens the door to a systematic study of nonlocal Lagrangian formalisms, symmetry principles, and their consequences in deep learning.

\section{Numerical illustration of nonlocal dynamics}\label{sec:Numerical_Simulations}

In this section we provide numerical illustrations of the nonlocal dynamics derived in Section~\ref{Sec:Adam-dynamics} and of its first-order reduction. The goal is twofold: first, to benchmark the proposed second-order inertial continuous-time model against the standard discrete Adam optimizer; and second, to isolate the effect of the inertial correction by comparing it with the corresponding first-order nonlocal flow.  As shown, the simulations mirror the behavior of the discrete algorithm, confirming that the second-order integro-differential equation serves as an accurate and reliable model for analyzing the dynamics of these optimization methods.

Throughout, we compare discrete Adam iterates $\{\theta_k\}_{k\ge 0}$ with the continuous-time trajectories evaluated at matched times $t_k := k\alpha$. This is consistent with the scaling $t \approx k\alpha$ used in the derivation of the nonlocal dynamics (Section~\ref{Sec:Adam-dynamics}).

Furthermore, we consider the one-dimensional Rosenbrock-type loss \cite{Rosenbrock1960AutomaticMethod}
\begin{equation}\label{eq:rosenbrock_1d}
f(\theta) \;=\; (1-\theta)^2 \;+\; c\,(\theta^2-1)^2, \qquad c  \geq 0\,,
\end{equation}
with gradient
\begin{equation}\label{eq:grad_rosenbrock_1d}
\nabla f(\theta) \;=\; 2(\theta-1) \;+\; 4\,c\,\theta(\theta^2-1).
\end{equation}
This objective function provides a simple nonlinear setting whose degree of nonconvexity is controlled by the parameter \(c\). More precisely, \(c\) tunes the strength of the quartic term \(c(\theta^2-1)^2\), which promotes two preferred regions near \(\theta=\pm 1\). 
For small \(c\), the quadratic component \((1-\theta)^2\) dominates and the landscape is effectively single-well, with a unique minimizer at \(\theta=1\).  For larger \(c\), a double-well geometry emerges, with two basins separated by a barrier near the origin (Figure~\ref{fig:f-theta}). 

In our study we focus on two representative values, $c=1.5$ and $c=4$. The derivative of~\eqref{eq:grad_rosenbrock_1d} factorizes as
\[
f'(\theta)=2(\theta-1)\bigl(2\,c\,\theta^2+2\,c\,\theta+1\bigr).
\]
Hence, for $c<2$ the only critical point is the global minimizer $\theta^*=1$, while for $c>2$ two additional critical points appear on the negative side,
\[
\theta_\pm=\frac{-1\pm\sqrt{1-2/c}}{2}<0,
\]
corresponding to a local maximum and a secondary local minimum. We use $c=1.5$ as a mildly nonconvex but unimodal case, and $c=4$ as a genuinely bistable (double-well) landscape with a metastable negative basin.

\begin{figure}[h!]
\centering
\includegraphics[width=0.6\linewidth]{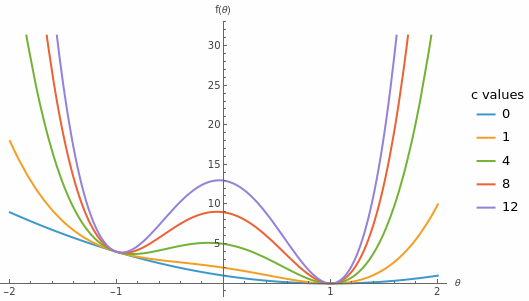}
\caption{Rosenbrock-type loss for several values of \(c\).}
\label{fig:f-theta}
\end{figure}

Before presenting the results, we provide an overview of the numerical method employed to solve the first- and second-order integro-differential equations.

\subsection{Numerical Method}

To generate trajectories for the causal integro-differential dynamics derived in Section~\ref{Sec:Adam-dynamics}, we solve the resulting IDE using a fixed-point (Picard-type) iteration~\cite{Tricomi1985IntegralEquations} based on the \textsc{IDESolver} strategy of~\cite{GELMI2014}. We implement the method from scratch and adapt it to our setting (in particular, to the strictly causal memory terms and to the discretization choices used throughout this work), rather than relying on the available Python module. Pseudocode is provided in Algorithm~\ref{alg:IDESolver}.

At a high level, the method alternates between (i) freezing the nonlocal term using the previous iterate and (ii) solving the resulting local-in-time ODE to update the trajectory. Concretely, we first construct an initial guess $y_{\text{guess},0}$ by dropping the integral (nonlocal) contribution and integrating the remaining ODE. Given an iterate $y_{\text{guess},n}$, we then evaluate the nonlocal term (the memory/convolution contribution) using $y_{\text{guess},n}$ and integrate the resulting ODE to obtain a provisional update $y_{\text{current}, n}$. Convergence is monitored on the discretization grid by the squared $\ell^2$-difference between successive iterates:
\[
\text{error} = \sum_{n=1}^{M} \left( y_{\text{current},n} - y_{\text{guess},n} \right)^2,
\]
where \( M \) denotes the total number of discretization points,  \( y_{\text{current},n} \) represents the current estimate of \( y \),  and \( y_{\text{guess},n} \) is the guess, both evaluated at the \( n \)-th step.

To improve robustness, we use an update:
\[
y_{\text{new}} = a\,y_{\text{current}} + (1 - a) y_{\text{guess}}.
\]
with an adaptive relaxation parameter $a\in[0.5,0.9999]$ chosen based on the observed error decrease. If the error increases from one iteration to the next, we increment $a$ by $5\times 10^{-4}$ to promote convergence. We declare convergence when $\text{error} \le 10^{-4}$ and use a very large iteration cap (set to $10^{3}$) as a safety fallback.

To evaluate the memory integrals, we use Gaussian quadrature with $n=10^3$ nodes\footnote{We use such a high number to be accurate with the numerical simulation.} (instead of the \texttt{quad} routine in~\cite{GELMI2014}), which significantly reduces runtime while preserving accuracy for the kernels considered here. For the local ODE update, we initially employed an explicit Euler scheme in order to remain consistent with the uniform-grid Euler discretization assumed in our derivation,
\[
t \in \{t_0, t_0+\alpha, t_0+2\alpha, \dots\}, \qquad t < t_{\text{end}},
\]
and the discrete-to-continuous matching $t_k = k\alpha$ used when comparing continuous trajectories to discrete Adam iterates.

In the second-order formulation, however, the velocity component $u^i(t):=\dot \theta^i(t)$ satisfies a fast linear relaxation of the form \(\dot u^i(t)=-(2/\alpha)\bigl(u^i(t)+F^i(t)\bigr)\), where $F^i(t)$ collects the nonlocal forcing terms induced by the Adam moments. Under explicit Euler, the homogeneous part updates as \(u^i_{k+1}=(1-2h/\alpha)\,u^i_k\); thus, taking $h=\alpha$ places the method at the stability boundary with amplification factor $-1$, which excites a false step-to-step sign alternation (a ``zig-zag'' mode) that appears as persistent high-frequency oscillations in the velocity and, even more prominently, in the acceleration (which inherits an additional factor $2/\alpha$). To eliminate this numerical artifact while preserving the model dynamics, we refine the integration step and use $h=\alpha/5$, for which \(1-2h/\alpha=0.6\), yielding a contractive update of the fast mode and producing smooth, physically consistent decay in the velocity and acceleration trajectories. Moreover, we set the numerical stabilizer to $\epsilon = 10^{-8}$.

When the dynamics requires evaluating delayed/shifted terms (e.g.\ the ``future'' increment parameter $\alpha$ in Proposition~\ref{prop:Adam2}), we obtain the necessary off-grid values via cubic interpolation using
\texttt{scipy.interpolate.interp1d}.

The above loop in Algorithm~\ref{alg:IDESolver} is repeated until the stopping (tolerance) criterion is met. Implementation details (including parameters and
reproducibility scripts) are provided in the code repository\footnote{Code available at \href{https://github.com/carlosherediapimienta/nonlocal-adam}{GitHub}.}.

\begin{algorithm}[t!]
\caption{Iterative Modified IDESolver Method}
\label{alg:IDESolver}
\begin{algorithmic}[1]
\State Initialize the iteration counter $k \gets 0$
\State Compute the initial solution $y_{\text{guess}}$ using the original differential equation
\State Compute the initial guess $y_{\text{current}}$ including the integral part with $y_{\text{guess}}$
\State Calculate the initial global error: $\text{error} \gets \| y_{\text{current}} - y_{\text{guess}} \|$
\While{$\text{error} > \text{tolerance}$}
    \State Compute new solution $y_{\text{new}}$ using a smoothing factor with $y_{\text{current}}$ and $y_{\text{guess}}$
    \State Update $y_{\text{guess}}$ solving the ODE including the integral part with $y_{\text{new}}$
    \State Calculate the current global error: $\text{new\_error} \gets \| y_{\text{new}} - y_{\text{guess}} \|$
    \If{$\text{new\_error} > \text{error}$}
        \If{maximum smoothing factor reached}
            \State Exit the loop without achieving the desired tolerance
        \Else
            \State Update the smoothing factor to the next value
        \EndIf
    \EndIf
    \State Update $y_{\text{current}} \gets y_{\text{new}}$
    \State Increment the iteration counter $k \gets k + 1$
    \If{$k > k_{\text{max}}$}
        \State Exit the loop
    \EndIf
    \State Update $\text{error} \gets \text{new\_error}$
\EndWhile
\State Set the final solution $y \gets y_{\text{guess}}$
\State \Return time values and the corresponding solution $y$
\end{algorithmic}
\end{algorithm}

\subsection{Discrete-to-continuous approximation accuracy}
To quantify how well the continuous-time models reproduce the discrete Adam iterates, we compare the discrete trajectory $\{\theta_k\}_{k\ge 0}$ with the continuous trajectories sampled at matched times $t_k := k\alpha$. Fixing a final physical time $T>0$ and setting $K=\lfloor T/\alpha\rfloor$, we consider the final-time discrepancies
\[
E_T(\alpha) := \bigl|\theta_K - \theta(T)\bigr|,
\qquad
E_{f,T}(\alpha) := \bigl|f(\theta_K) - f(\theta(T))\bigr|.
\]
We report $E_T(\alpha)$ and $E_{f,T}(\alpha)$ for both the first-order and second-order continuous-time models over a range of step sizes $\alpha$, while keeping $(\beta_1,\beta_2)$, $\theta_0$, and the horizon $T$ fixed. Figure~\ref{fig:error_scaling} shows the resulting log-log trends for $c=1.5$.

Across the tested step sizes, the performance of the inertial (second-order) model depends on the chosen initial velocity $u_0$ at wide discretizations. For relatively large $\alpha$, we observe that positive initial velocities yield a more accurate final-time match than the first-order model, whereas $u_0=0$ or negative values can lead to a slightly worse final discrepancy. As the learning rate is decreased and enters the small-step regime (here, for $\alpha\lesssim 10^{-3}$), the second-order model becomes consistently more accurate than the first-order surrogate across the tested velocities. This transition is consistent with the theoretical picture: for fixed $(\beta_1,\beta_2) = (0.99,0.999)$, the inertial model constitutes an $\alpha$-refinement of the first-order one, so improved agreement with the discrete iterates is expected once $\alpha$ is sufficiently small. In addition, since the continuous-time dynamics are numerically integrated with a step $h=\alpha/5$, decreasing $\alpha$ also reduces the ODE solver error, making the model-refinement effect more clearly visible in this regime.

To facilitate a compact comparison of how errors decay, we annotate each curve with an empirical log-log slope $p$ obtained from a least-squares fit of the form $E(\alpha)\approx C\,\alpha^{p}$ over the tested $\alpha$-range (reported in the legend for readability). We interpret $p$ as an effective slope on this finite range (rather than a formal asymptotic order), since final-time errors may be dominated by pre-asymptotic behavior and by the combined effects of model mismatch and the numerical solver used to approximate the continuous-time dynamics. Importantly, the figure exhibits a clear separation between state and objective discrepancies: the slopes for $E_{f,T}$ are approximately twice those for $E_T$. This behavior is expected near a nondegenerate minimizer $\theta^*$, where $f'(\theta^*)=0$ and a Taylor expansion gives
\[
f(\theta_K)-f(\theta(T)) \;=\; \tfrac12 f''(\xi)\,(\theta_K-\theta(T))^2,
\]
for some $\xi$ between $\theta_K$ and $\theta(T)$. Hence, once both trajectories are close to $\theta^*$, objective discrepancies are generically quadratic in the state discrepancy, i.e., $E_{f,T}(\alpha)\sim \mathcal{O}(E_T(\alpha)^2)$, which explains the observed relation $p_f\approx 2p$.

Finally, these observations align with the theoretical picture developed earlier: with $(\beta_1,\beta_2)$ fixed, the inertial second-order model constitutes an $\alpha$-refinement of the first-order dynamics. In this sense, the second-order correction captures higher-fidelity features of the discrete-time Adam dynamics at finite step sizes, consistent with the systematically smaller final discrepancies reported in Figure~\ref{fig:error_scaling}.

\begin{figure}[h!]
\centering
\includegraphics[width=0.85\linewidth]{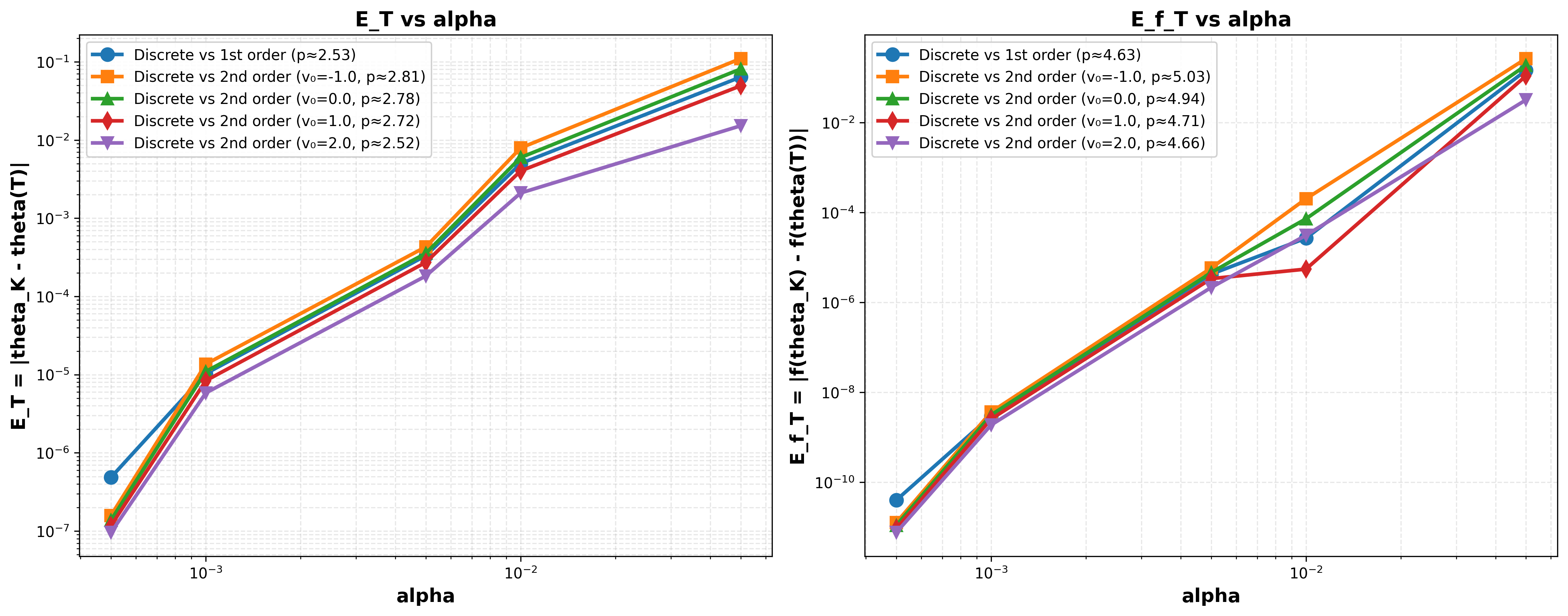}
\caption{
Final-time errors between discrete Adam $\theta_K$ (with $K=\lfloor T/\alpha\rfloor$) and the continuous-time model sampled at $t_k=k\alpha$: $E_T(\alpha)=|\theta_K-\theta(T)|$ (left) and $E_{f,T}(\alpha)=|f(\theta_K)-f(\theta(T))|$ (right) for $\beta_1=0.99$ and $\beta_2=0.999$. The legend reports an effective log-log slope $p$ from a fit $E(\alpha)\approx C\alpha^p$.
Across step sizes, the inertial second-order model is consistently closer to the discrete iterates, in line with the theory; near the minimizer, $E_{f,T}=\mathcal{O}(E_T^2)$ explains the larger slopes in the right panel.
}
\label{fig:error_scaling}
\end{figure}

\subsection{Numerical Dynamics Results: case $c=4$}\label{subsec:Numerical_3}

We now turn to the genuinely nonconvex (bistable) regime $c=4$, where the objective landscape presents two competing basins of attraction and the dynamics can exhibit stronger transient effects. Unless stated otherwise, we fix the learning rate $\alpha=10^{-3}$ and the momentum parameters $(\beta_1,\beta_2)=(0.99,0.999)$, and we initialize the discrete and continuous-time models with the same $\theta_0$. 

Figure~\ref{fig:complete_optimization} highlights the inertial nature of the second-order continuous-time dynamics through the evolution of position, velocity, and acceleration. In this experiment we initialize $\theta(0)=2$ and a positive initial velocity $\dot\theta(0)=u_0=1$, so the trajectory starts with nontrivial kinetic energy and a pronounced early-time acceleration (an initial transient) as the momentum variables rapidly adjust. Consistent with the mechanical interpretation of the model as a damped inertial system driven by an adaptive force, the velocity quickly relaxes and then decays toward zero, while the acceleration becomes small once the trajectory enters the long-time relaxation regime. As a result, the position $\theta(t)$ is steered toward the stable minimizer at $\theta=1$ (the expected basin for this initialization), and the phase portrait $(\theta,\dot\theta)$ together with the decay of the kinetic energy $\tfrac12|\dot\theta|^2$ illustrates the dissipative approach to equilibrium. In parallel, the moments $(m,v)$ build up in response to the large initial gradient and subsequently relax as $\nabla f(\theta)$ vanishes, causing the effective adaptive update scale $|m|/(\sqrt{v}+\varepsilon)$ to decrease by several orders of magnitude, which explains the progressive slowdown of the dynamics near the minimizer.

\begin{figure}[h!]
\centering
\includegraphics[width=0.85\linewidth]{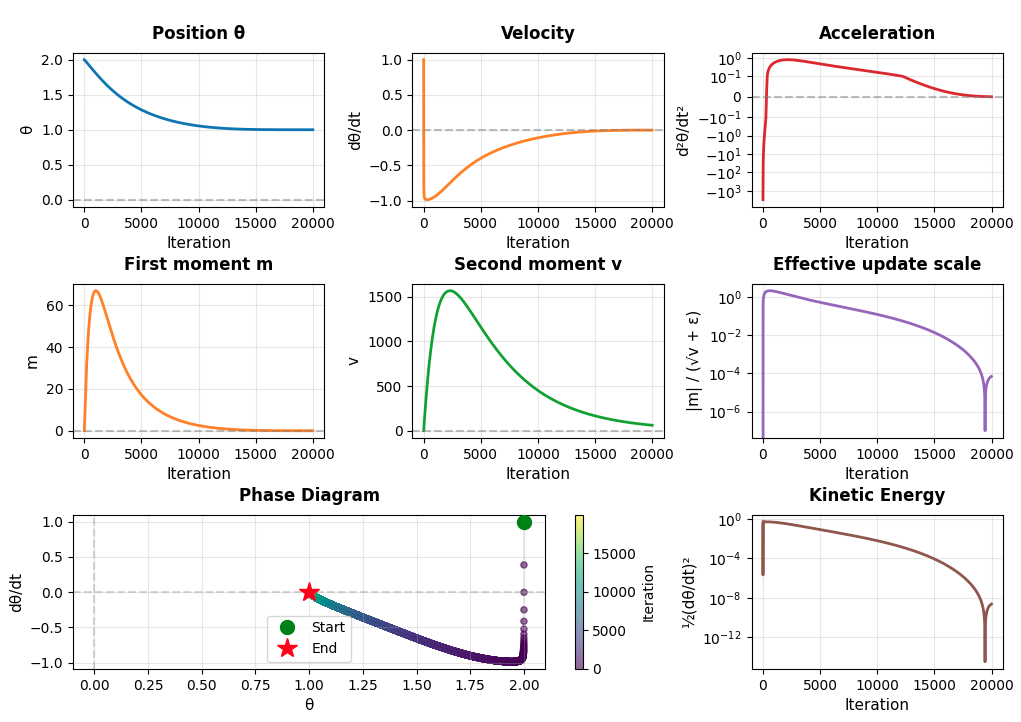}
\caption{
Complete optimization dynamics for $c=4$ with $\alpha=10^{-3}$ and $(\beta_1,\beta_2)=(0.99,0.999)$, initialized at $\theta(0)=2$ and $\dot\theta(0)=u_0=1$. Top row: position $\theta$, velocity $\dot\theta$, and acceleration $\ddot\theta$, illustrating an early inertial transient followed by dissipative relaxation. Middle row: first and second moments $(m,v)$ and the effective adaptive update scale $|m|/(\sqrt{v}+\varepsilon)$.
Bottom row: phase portrait $(\theta,\dot\theta)$ (colored by iteration) and kinetic energy $\tfrac12|\dot\theta|^2$. The trajectory converges toward the minimizer at $\theta=1$, consistent with the expected basin of attraction for this initialization.
}
\label{fig:complete_optimization}
\end{figure}

Figure~\ref{fig:transition_c4} illustrates a basin-selection transition in the bistable case $c=4$ as the stepsize $\alpha$ varies. We fix $(\beta_1,\beta_2)=(0.99,0.999)$ and initialize at $\theta_0=-1.5$; for the inertial second-order model we set $u_0=1$ (and observe the same qualitative behavior for the other tested choices of $u_0$). For relatively large learning rates, both the discrete Adam iterates and the second-order continuous-time model converge to the global minimizer at $\theta=1$, whereas as $\alpha$ is decreased the dynamics remains trapped in the basin of the local minimizer (here $\theta\approx -0.853$). Importantly, the second-order model reproduces the same direction of the transition and the same threshold-like jump in the final state, indicating that it captures the finite-stepsize mechanism responsible for basin switching in this nonconvex landscape.

\begin{figure}[h!]
\centering
\includegraphics[width=0.75\linewidth]{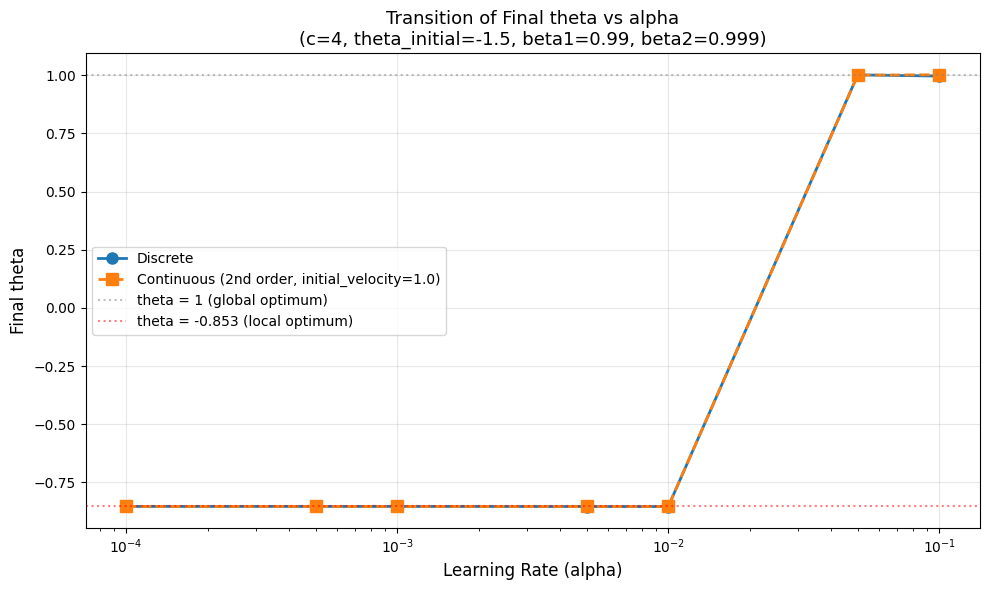}
\caption{Basin-selection transition in the bistable case $c=4$ as the stepsize $\alpha$ varies. We fix $(\beta_1,\beta_2)=(0.99,0.999)$ and initialize at $\theta_0=-1.5$; for the inertial second-order model we set $u_0=1$. The plot reports the final value of $\theta$ as a function of $\alpha$ for discrete Adam and for the second-order continuous-time model. Both exhibit the same threshold-like basin switch: for relatively large $\alpha$ they converge to the global minimizer $\theta=1$, whereas for smaller $\alpha$ the dynamics remains trapped near the local minimizer (here $\theta \approx -0.853$).
}
\label{fig:transition_c4}
\end{figure}

\subsection{Convex case: case $c=0$}
We complement the experiments with the convex baseline $c=0$, for which~\eqref{eq:rosenbrock_1d} reduces to the quadratic loss $f(\theta)=(1-\theta)^2$, with unique minimizer $\theta^\star=1$ and $f_\star=0$. Beyond serving as a simple reference landscape, the purpose of this case is methodological: it provides a controlled regime where the long-time behavior is unambiguous and the dynamics are not affected by basin selection. This makes $c=0$ an ideal test bench to validate, at the numerical level, the exponential-plus-offset bound predicted by the Lyapunov analysis of the nonlocal dynamics, namely a decay of the form $A e^{-\omega(t-\delta)}+B$ after an initial time, Theorem~\ref{thm:robust-inertial-adam}.

We fix $(\beta_1,\beta_2)=(0.99,0.999)$ and run discrete Adam with stepsize $\alpha=10^{-3}$ from $\theta_0=-1.5$. We compare it against the inertial second-order continuous-time model initialized with the same initial position and velocity $u_0=1$, and we sample the continuous trajectory at matched times $t_k=k\alpha$. As in the rest of the section, the second-order IDE is integrated with a refined step $h=\alpha/5$ to suppress fast-mode oscillations in the velocity and acceleration components.

To robustly evaluate the decay rate, we study the objective residual $\Phi(t):=f(\theta(t))-f_\star\ge 0$ through its running upper envelope $\mathrm{env}(\Phi)$ computed on a short physical-time window. We introduce this envelope in order to factor out the small-amplitude oscillations produced by the inertial correction and by discretization effects, and to focus on the monotone decay trend that is relevant for the exponential$+$offset rate profile. Here, the envelope is defined as a backward running maximum over a fixed window of length $w>0$ in physical time:
\[
\mathrm{env}(\Phi)(t)\;:=\;\max\bigl\{\Phi(s): s\in[t-w,\,t]\bigr\}.
\]
In practice, we evaluate $\mathrm{env}(\Phi)$ on the discrete time grid $\{t_n\}$ by taking, at each $t_n$, the maximum of $\Phi(t_m)$ over all indices $m\le n$ such that $t_m\ge t_n-w$. We then test whether $\mathrm{env}(\Phi)$ admits an upper bound of the form
\begin{equation}\label{eq:exp_offset_bound}
\mathrm{env}(\Phi)(t)\ \le\ A\,e^{-\omega (t-\delta)} + B,
\end{equation}
where $\delta>0$ denotes an initial-time cutoff. Concretely, for each trajectory (discrete and second-order continuous) we: (i) fix $\delta$ and compute $\mathrm{env}(\Phi)$; (ii) estimate an offset $B$ from a final tail window; (iii) fit an exponential rate $\omega$ over the intermediate regime $[\delta,\,T_{\mathrm{end}}-\text{(tail)}]$; and (iv) choose an amplitude $A$ so that~\eqref{eq:exp_offset_bound} upper-bounds $\mathrm{env}(\Phi)$ over the full time range\footnote{For more details, see the notebook: notebooks/error\_phi.ipynb at \href{https://github.com/carlosherediapimienta/nonlocal-adam}{GitHub}}.

Figure~\ref{fig:bound_c0} reports the resulting envelope curves and the fitted exponential$+$offset bounds for both the discrete Adam trajectory and the inertial second-order model. In the left panel, the two envelopes essentially coincide across the full time horizon, indicating that the second-order nonlocal model faithfully reproduces the decay profile of the discrete iterates in this convex setting. Moreover, the bound~\eqref{eq:exp_offset_bound} captures the expected ``fast initial relaxation + slower long-time regime'' pattern: after the initial layer $t\approx\delta$, $\mathrm{env}(\Phi)$ decays approximately exponentially until it becomes small compared to the offset term.

The right panel of Figure~\ref{fig:bound_c0} shows the bound-violation diagnostic $\mathrm{env}(\Phi)(t)-\big(Ae^{-\omega(t-\delta)}+B\big)$. No positive values are observed, confirming that~\eqref{eq:exp_offset_bound} provides a valid upper envelope bound on the simulated trajectories. The inset zoom highlights that the curve approaches zero only in a narrow intermediate time range (where the fitted exponential is closest to the empirical decay), while it remains strictly negative elsewhere, reflecting the conservative nature of the global upper-bound construction.

\begin{figure}[!htb]
\centering
\includegraphics[width=0.85\linewidth]{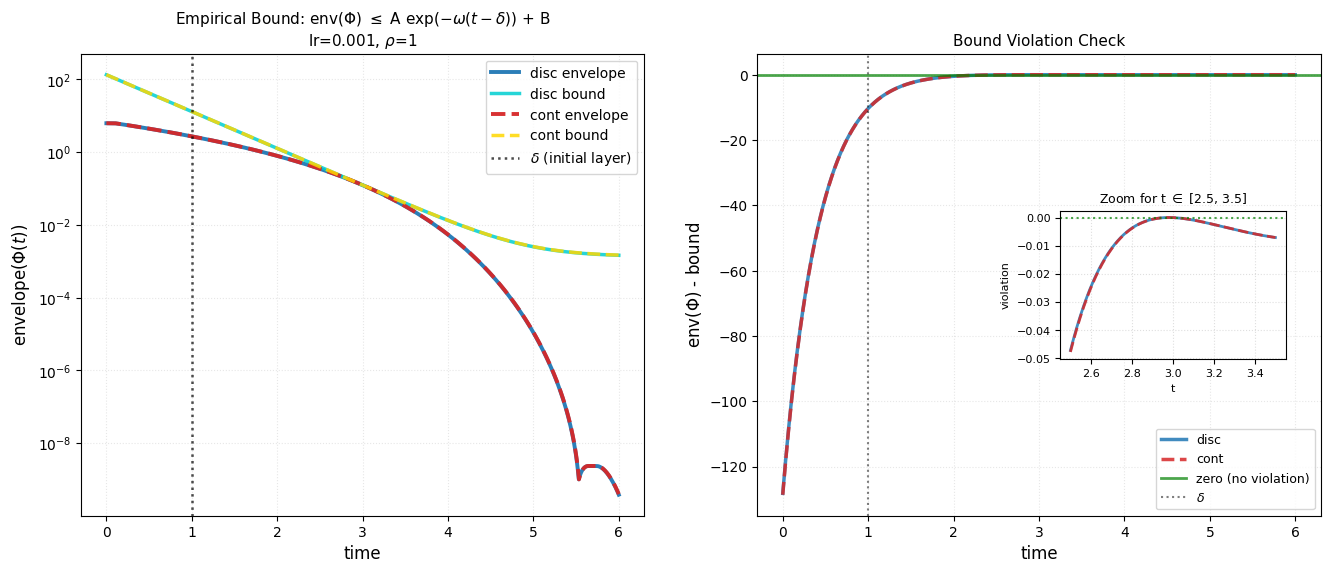}
\caption{
Convex baseline ($c=0$): validation of an exponential-plus-offset envelope bound for $\alpha=10^{-3}$ and $(\beta_1,\beta_2)=(0.99,0.999)$.
Fits: discrete Adam $(A,\omega,B)=(13.048,2.333,1.352\times10^{-3})$ and inertial second-order nonlocal model $(13.097,2.332,1.370\times10^{-3})$.
Left: $\mathrm{env}(\Phi)$ and fitted bounds (continuous model sampled at $t_k=k\alpha$).
Right: residual $\mathrm{env}(\Phi)-\text{bound}$ (inset zoom), with no positive violations.
}
\label{fig:bound_c0}
\end{figure}

\subsection{Initial velocity}\label{subsec:initial_velocity}

Because the inertial model is second order, it requires prescribing an initial velocity $u_0=\dot\theta(0)$, which has no discrete-time analogue in Adam. To quantify the impact of this additional initial condition on the relaxation regime, we compare trajectories on a post-initial-time window $t\in[\delta,T]$. In the experiments below we fix $c=1.5$, $(\beta_1,\beta_2)=(0.99,0.999)$, and $\theta_0=-1.5$, and we consider several choices $u_0\in\{-1,1,2\}$ for the inertial model.

We evaluate the second-order dynamics at matched times $t_k=k\alpha$ and set $\delta =1$. Taking a reference velocity $u_0^{\rm ref}=0$, we define the post-transient sup-norm discrepancies
\[
\begin{aligned}
\Delta_\theta(u_0)\;&:=\;\sup_{t\in[\delta,\,\min\{T,t_{\rm end}\}]}\bigl|\theta_{u_0}(t)-\theta_{u_0^{\rm ref}}(t)\bigr|,\\
\Delta_\Phi(u_0)\;&:=\;\sup_{t\in[\delta,\,\min\{T,t_{\rm end}\}]}\bigl|f(\theta_{u_0}(t))-f(\theta_{u_0^{\rm ref}}(t))\bigr|.
\end{aligned}
\]

Figure~\ref{fig:velocity} reports $\Delta_\theta(u_0)$ and $\Delta_\Phi(u_0)$ as functions of the stepsize $\alpha$. We observe that the dependence on $u_0$ is mild in the refined (small-$\alpha$) regime: as $\alpha$ decreases, both discrepancies drop rapidly, indicating that after the initial time the trajectories synchronize and the long-time relaxation profile is largely insensitive to the prescribed velocity. Sensitivity increases at larger discretizations, consistent with a stronger influence of inertial initialization outside the small-step regime. Finally, the effect is predominantly controlled by the magnitude of the velocity perturbation: the curves for $u_0=-1$ and $u_0=1$ are nearly indistinguishable, while the larger choice $u_0=2$ yields systematically higher deviations.

\begin{figure}[!htb]
\centering
\includegraphics[width=0.85\linewidth]{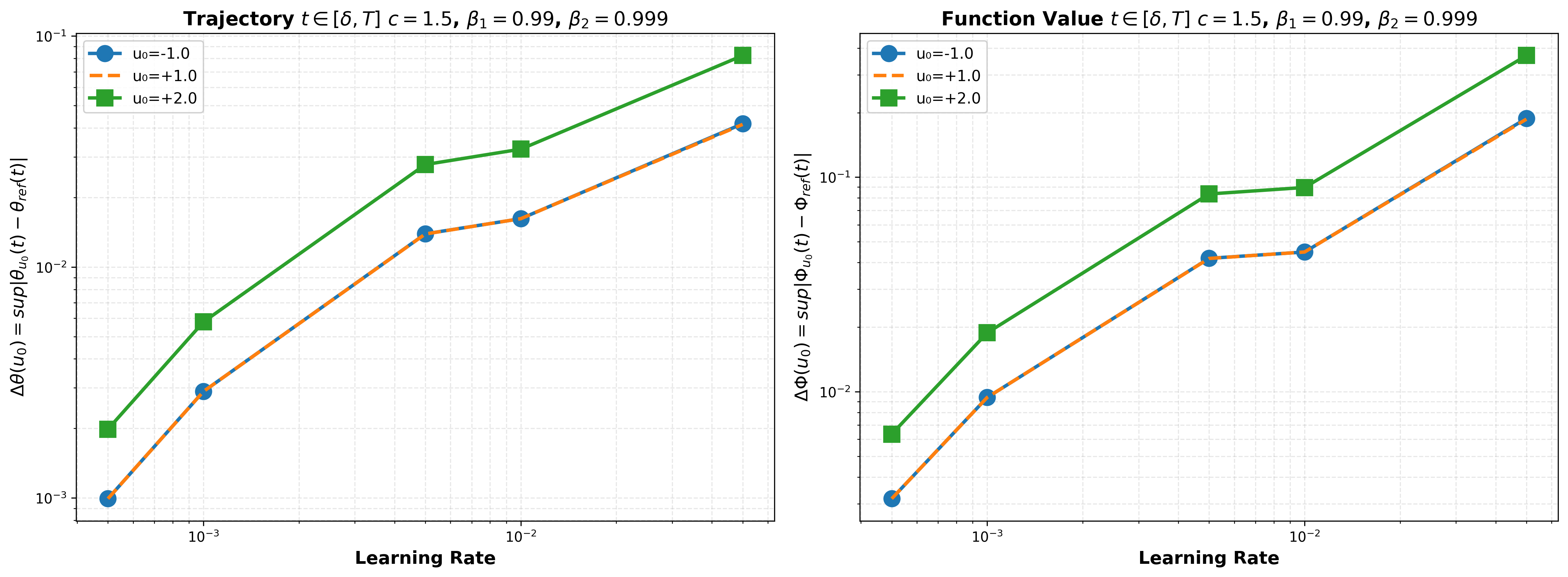}
\caption{
Dependence on the initial velocity $u_0$ in the inertial second-order model (case $c=1.5$) measured on the post-transient window $t\in[\delta,T]$. Left: $\Delta_\theta(u_0)=\sup_{t\in[\delta,\min\{T,t_{\rm end}\}]}|\theta_{u_0}(t)-\theta_{u_0^{\rm ref}}(t)|$. Right: $\Delta_\Phi(u_0)=\sup_{t\in[\delta,\min\{T,t_{\rm end}\}]}|f(\theta_{u_0}(t))-f(\theta_{u_0^{\rm ref}}(t))|$.
}
\label{fig:velocity}
\end{figure}

\subsection{Critical beta regimes and moment-positivity diagnostics}\label{subsec:betas_critical}

In addition to the standard high-memory setting $(\beta_1,\beta_2)=(0.99,0.999)$, we also explore more delicate choices of $(\beta_1,\beta_2)$, such as $\beta_1=0.5$ and $\beta_2<0.5$, in order to stress-test the numerical procedure and the discrete-to-continuous agreement outside the usual parameter regime. A key numerical subtlety in these runs is the structural constraint $v(t)\ge 0$ for the second-moment variable: during the Picard fixed-point iterations used to solve the IDE, intermediate iterates may produce a provisional profile with $v(t)<0$ at isolated times. When this occurs, we enforce the constraint by projecting onto the nonnegative $v(t)\ \leftarrow\ \max\{v(t),0\}$ before evaluating the subsequent update. Empirically, this simple projection does not prevent convergence of the Picard scheme; rather, it stabilizes the iteration while preserving the expected qualitative behavior. In the converged solution used for the comparisons below, the second moment is nonnegative, consistent with the intended interpretation of $v$ as a squared-gradient accumulator.

Figure~\ref{fig:betas_dynamics} reports a complete dynamical decomposition for a representative critical setting (here $\alpha=10^{-2}$, $\beta_1=0.5$, $\beta_2=0.2$, $c=1.5$, and $\theta_0=-1.5$), including the evolution of $\theta$, $\dot\theta$, and $\ddot\theta$, the moments $(m,v)$, and the effective update scale $|m|/(\sqrt{v}+\varepsilon)$. In this short-memory regime, the moments build up and relax on a time scale comparable to the post-initial-time dynamics, yielding a pronounced transient in the velocity and acceleration. Importantly, as can be seen in the simulations, once the solver has converged the resulting second moment remains strictly nonnegative throughout the trajectory, i.e., $v(t)\ge 0$ for all reported times. 

In addition, these diagnostics provide a direct stability check of the numerical dynamics in a challenging beta regime. Here we intentionally work in a critical configuration with $\beta_1 \nleq \sqrt{\beta_2}$. As can be seen in the plots, the instability of this regime is directly visible through sustained oscillations in the velocity $\dot\theta(t)$ and, even more clearly, in the acceleration $\ddot\theta(t)$; the kinetic energy $\tfrac12|\dot\theta(t)|^2$ mirrors the same oscillatory, non-monotone pattern instead of a clean dissipative decay.

Despite these more critical choices of $(\beta_1,\beta_2)$, the resulting continuous-time trajectories remain fully consistent with the discrete Adam dynamics: as we show next, both the qualitative relaxation pattern and the quantitative evolution of $\theta(t)$ and $\Phi(t)$ stay closely aligned with the corresponding discrete iterates.

\begin{figure}[!htb]
\centering
\includegraphics[width=0.85\linewidth]{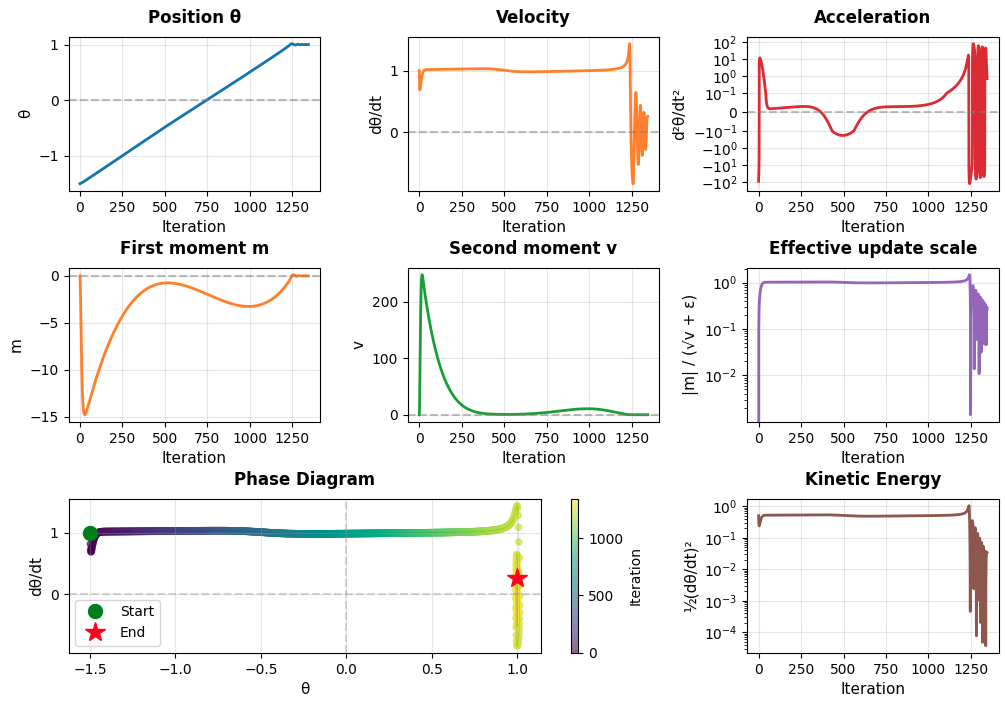}
\caption{
Critical beta regime: evolution of the second-moment variable $v(t)$ in the inertial continuous-time model. The reported (converged) trajectory satisfies the structural constraint $v(t)\ge 0$ for all times; in the implementation we enforce this by clipping, $v\leftarrow \max\{v,0\}$, whenever needed during the numerical iterations.
}
\label{fig:betas_dynamics}
\end{figure}

Figures~\ref{fig:betas_convergence}-\ref{fig:betas_theta} compare the discrete Adam iterates with the first-order and second-order continuous-time models at matched times $t_k=k\alpha$. Despite the stronger transients induced by the critical choices of $\beta_1$ and $\beta_2$, the continuous-time trajectories remain quantitatively close to the discrete optimizer throughout the run. Moreover, the late-time regime (highlighted by the zoomed insets) shows that all three dynamics converge to the same minimizer and that the second-order model continues to provide an accurate refinement of the first-order approximation. Taken together, these results indicate that even in short-memory, more delicate parameter regimes, the proposed second-order nonlocal model remains well behaved and reproduces the correct terminal behavior.

\begin{figure}[!htb]
\centering
\includegraphics[width=0.75\linewidth]{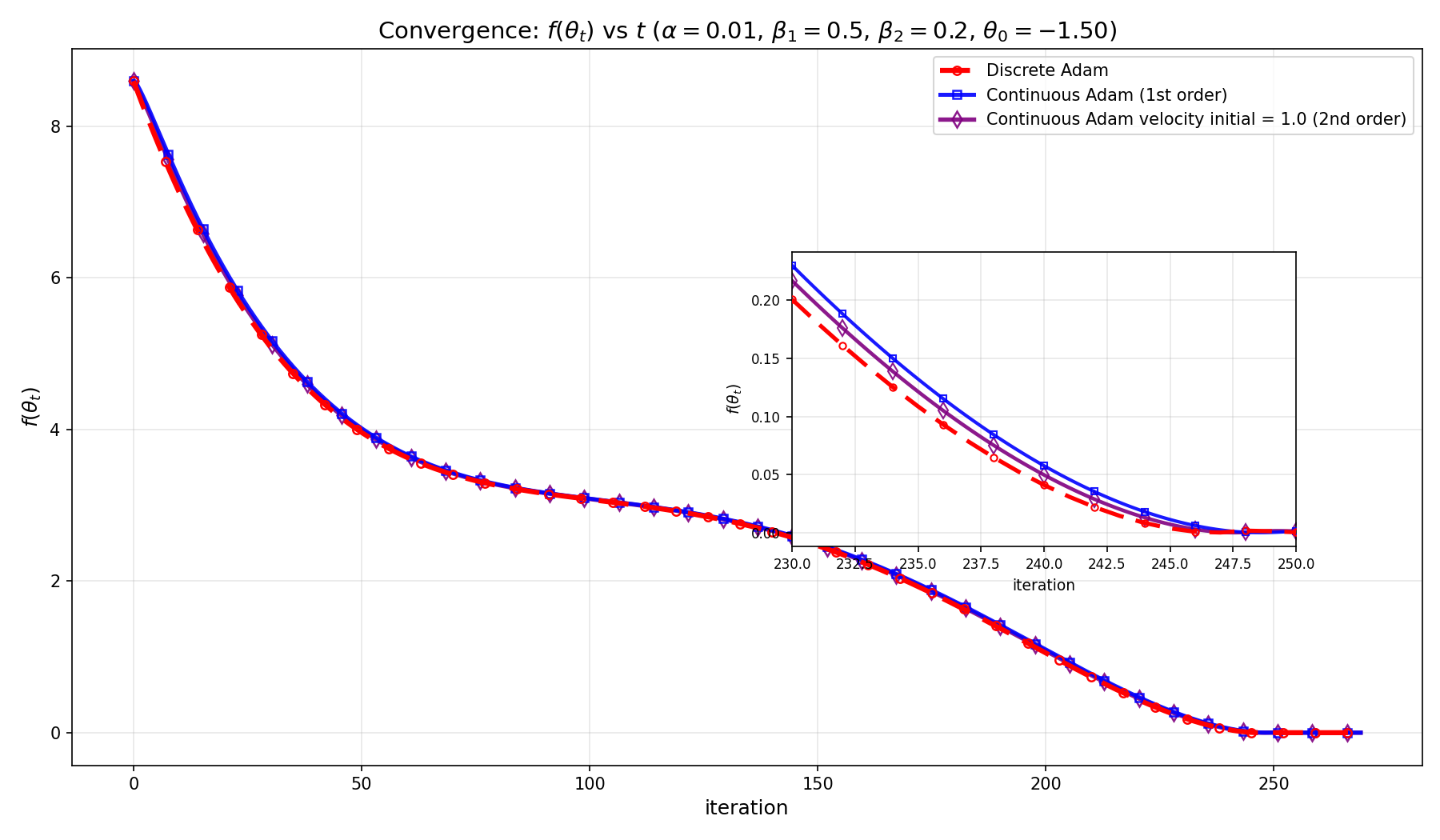}
\caption{
Critical beta regime ($\alpha=10^{-2}$, $\beta_1=0.5$, $\beta_2=0.2$, $c=1.5$, $\theta_0=-1.5$). Despite the challenging choice $\beta_1 \nleq \sqrt{\beta_2}$, the continuous-time models remain closely aligned with the discrete Adam iterates. This figure compares objective convergence (discrete vs. continuous-time surrogates).
}
\label{fig:betas_convergence}
\end{figure}

\begin{figure}[!htb]
\centering
\includegraphics[width=0.75\linewidth]{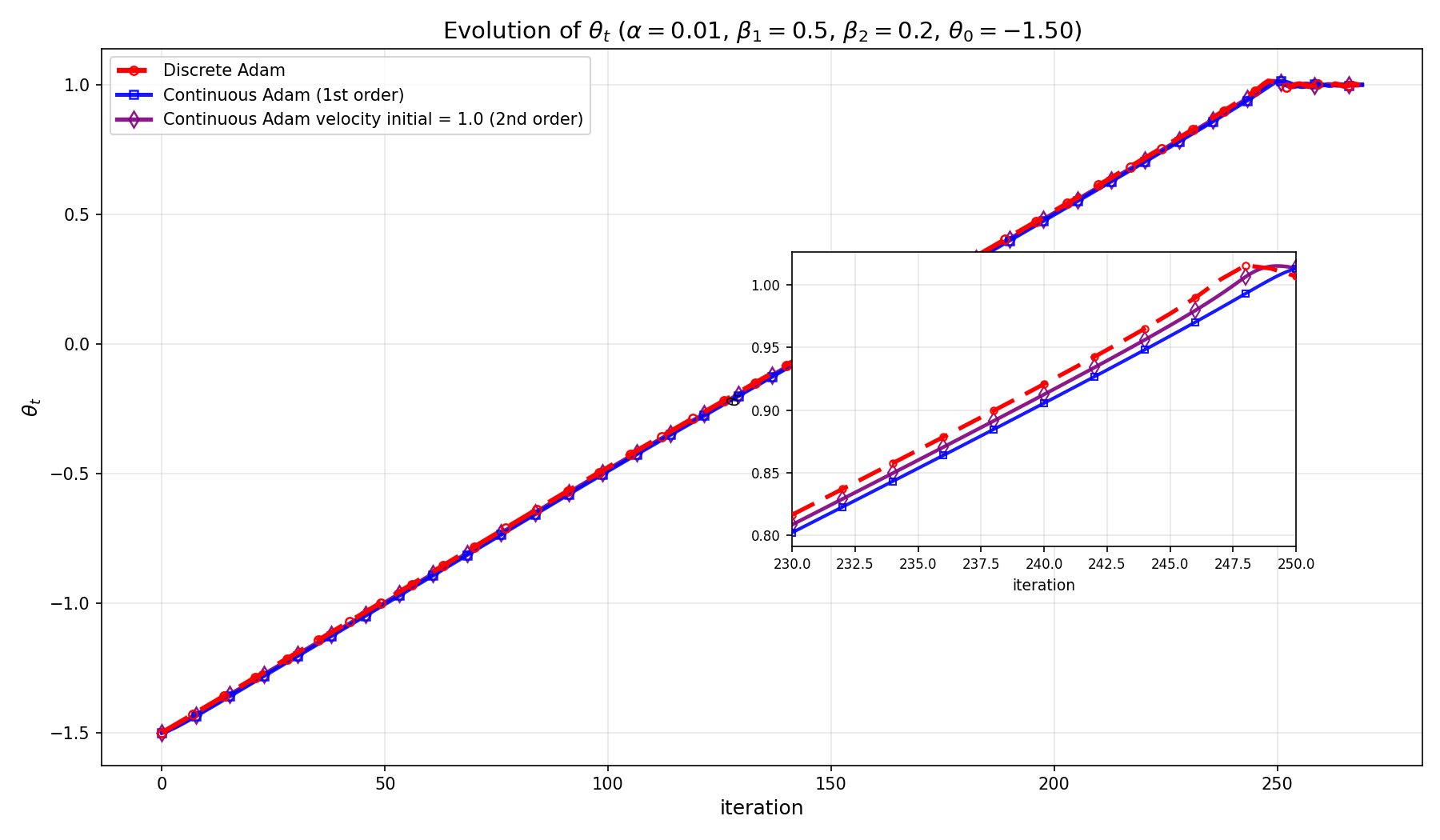}
\caption{
Critical beta regime ($\alpha=10^{-2}$, $\beta_1=0.5$, $\beta_2=0.2$, $c=1.5$, $\theta_0=-1.5$). Despite the challenging choice $\beta_1 \nleq \sqrt{\beta_2}$, the continuous-time models remain closely aligned with the discrete Adam iterates. This figure shows the state evolution $\theta(t)$ at matched physical times $t_k=k\alpha$ in the same regime.
}
\label{fig:betas_theta}
\end{figure}

\section{Conclusions}
This work develops a second-order, continuous-time dynamical interpretation of Adam. Starting from the discrete recursion---Algorithm~\ref{alg:Adam}---and performing a second-order expansion in the stepsize, we derive an inertial equation with linear friction whose driving term is nonlocal: it depends on the past trajectory through causal convolutions of the gradient and squared-gradient. This integro-differential formulation exposes the stepsize $\alpha$ as a memory scale and reveals regime-dependent kernel structures (critical/over/underdamped) induced by the exponential averaging parameters $(\beta_1,\beta_2)$.

A central theoretical outcome is that the accelerated second-order model is an $\alpha$-refinement of the first-order nonlocal Adam flow \cite{heredia2025modelingadagradrmspropadam}. Away from the initial time, the second-order moments and the resulting normalized force converge quantitatively to their first-order counterparts as $\alpha\to 0$, so inertia acts as a controlled perturbation rather than a new drift on fixed finite horizons $T$. The perturbation size is governed by
\[
\rho=\max\Bigl\{\alpha,\ \bigl(\alpha/(1-\beta_1)\bigr)^2,\ \alpha/(1-\beta_2)\Bigr\}.
\]
In addition, we address a structural constraint inherited from discrete Adam---nonnegativity of the second-moment variable---by providing a sufficient condition that guarantees positivity even in oscillatory (underdamped) regimes.

Building on the perturbative link to the first-order model, we establish stability and convergence guarantees for the inertial nonlocal dynamics via Lyapunov method \cite{DaSilva2020}. Under standard smoothness and compact-sublevel assumptions, and assuming the compatibility condition\footnote{which implies $\beta_1^2 \le \sqrt{\beta_2}$ that was also found in the original Adam  analysis \cite{kingma2014}.}
\[
\beta_1 \le \sqrt{\beta_2},
\]
the second-order flow inherits dissipation properties up to an $O(\rho)$ remainder. With additional geometric structure, this can be strengthened into quantitative convergence statements.  In the Polyak-\L{}ojasiewicz (PL) setting (including strong convexity), the objective residual admits an exponential decay estimate with an $O(\rho^2)$ residual term; likewise, in the Kurdyka-\L{}ojasiewicz (KL) setting, one obtains the corresponding KL rate bound up to a $\rho$-dependent remainder. Importantly, these $\rho$-terms arise as upper bounds in the Lyapunov analysis and should be interpreted as stability guarantees: they prevent the dynamics from destabilizing and certify convergence to a controlled neighborhood. 
They do not preclude the possibility that the residual actually converges to zero in specific problems or parameter regimes, as also suggested by our numerical simulations.

From a structural viewpoint, the inertial nonlocal equation resembles a forced mechanical system (inertia + friction + forcing), motivating a variational perspective. We introduce an Adam-inspired class of nonlocal Lagrangians whose Euler-Lagrange equations generate inertial, memory-dependent flows under a weighted reciprocity condition on the force functional. At the same time, we identify an obstruction to an exact Lagrangian formulation of strict Adam: strictly causal (one-sided) memory breaks the time-symmetry required for global reciprocity in kinematic space. This positions the proposed Lagrangian class as an ``ideal,'' reciprocity-preserving blueprint that can guide the design of new Adam-like algorithms via controlled symmetry breaking \cite{tanaka2021noether}.

Beyond the modeling insight, this viewpoint opens the door to importing the broader nonlocal Lagrangian and Hamiltonian formalism for this class of adaptive optimization dynamics, including nonlocal extensions of Noether-type results \cite{heredia2025nonlocalmechanics}. Such tools can, in principle, associate continuous symmetries with conserved (or quasi-conserved) quantities in the resulting nonlocal flows, providing structurally meaningful invariants. These objects can then be leveraged to better understand the long-time behavior of the dynamics and to inform principled modifications of the algorithm---e.g., by enforcing or approximating desirable invariances, tuning memory/friction mechanisms, or designing regularizations that promote stability and efficiency---thereby offering a systematic route toward improved adaptive methods.

Finally, numerical experiments on Rosenbrock-type families support the theory and validate the modeling choices. The second-order nonlocal surrogate tracks discrete Adam closely when trajectories are compared at matched physical times, and it becomes systematically more accurate than the first-order surrogate in the small-step regime, consistent with the $\alpha$-refinement interpretation. 
Moreover, the simulations provide a clear phenomenology of the stability condition $\beta_1 \le \sqrt{\beta_2}$: when this compatibility fails (i.e., $\beta_1 \nleq \sqrt{\beta_2}$), the accelerated dynamics exhibit pronounced oscillations in the velocity (and hence in the kinetic-energy-like quantity), signaling the onset of instability. In contrast, within the stable regime, we observe that for the representative example studied the long-time behavior is essentially invariant with respect to the additional initial condition introduced by the second-order model, namely the choice of initial velocity; different initial velocities lead to trajectories that rapidly synchronize after a short transient.  Finally, the experiments corroborate the structural constraints of the model by confirming that the second-moment variable remains positive along the flow, i.e., $v(t)>0$ throughout the evolution.

Overall, the paper provides:
\begin{itemize}
  \item a refined second-order continuous-time representation of Adam as an inertial nonlocal flow with explicit causal kernels;
  \item a quantitative bridge to first-order nonlocal Adam dynamics;
  \item Lyapunov-based stability and PL/KL convergence guarantees up to $\rho$-controlled neighborhoods;
  \item a variational Adam-like framework;
  \item numerical evidence that the model faithfully reproduces discrete behavior, including qualitative phenomena such as basin transitions.
\end{itemize}

Natural directions for future work follow the variational perspective developed here. 
A particularly promising route is to leverage nonlocal Noether-type principles to identify symmetry-induced conserved quantities. Such structural objects could provide principled design constraints and diagnostics, ultimately enabling the construction of new optimization algorithms that are Adam-inspired yet derived from symmetry and invariance considerations rather than heuristic modifications. To conclude, while we focused on Adam due to its widespread use in modern machine learning, the same dynamical and variational program can be carried out for other adaptive methods---including AdaGrad \cite{AdaGra} and RMSProp \cite{hinton_lecture6}---suggesting a broader framework for understanding and designing adaptive optimizers.

\section*{Acknowledgments}
To my darling Daniela: Thank you for being my constant inspiration and my greatest motivation.

\section*{Author Contributions}
Carlos Heredia was solely responsible for the conception, development, implementation, and writing of this manuscript.

\section*{Funding}
The author declares that no funding was received for this work.

\appendix

\section{Appendix: A uniform shift estimate for the first-order drift}\label{app:shift-estimate}

In the proof of Theorem~\ref{thm:robust-inertial-adam} we use that, away from the initial time, the first-order nonlocal drift is Lipschitz in time with a constant independent of the stepsize~$\alpha$. This yields the shift estimate
\[
\|F(t+\alpha)-F(t)\|\le C_{\delta,T}\,\alpha
\qquad \text{for } t\in[\delta,T-\alpha],
\]
which allows us to replace $F(t+\alpha)$ by $F(t)$ at the cost of an $O(\alpha)$ perturbation. For completeness we provide the argument.

\paragraph{Settings:} Fix $0<\delta<T$ and assume $0<\alpha<\delta$.
Let $g(t):=\nabla f(\theta(t))$ and assume
\[
\|g(t)\|\le G_{\max},\qquad 
\|\dot g(t)\|\le \dot G_{\max}\quad \forall t\in[0,T].
\]
Let $M_1$ be the first-order moment solving
\begin{equation}\label{eq:M1_eq_diff}
\dot M_1(t)+\lambda_1 M_1(t)=\lambda_1 g(t),\qquad M_1(0)=0,\qquad 
\lambda_1:=\frac{1-\beta_1}{\alpha}.    
\end{equation}
Let $M_2$ solve
\[
\dot M_2(t)+\lambda_2 M_2(t)=\lambda_2 \|g(t)\|^2,\qquad M_2(0)=0,\qquad
\lambda_2:=\frac{1-\beta_2}{\alpha}.
\]
We assume that for some constants $\varepsilon_0,\varepsilon_1>0$, and $\forall t\in[\delta,T]$,
\[
\varepsilon\in C^1([\delta,T]),\qquad \varepsilon(t)\ge \varepsilon_0,\qquad |\dot\varepsilon(t)|\le \varepsilon_1,
\qquad M_2(t)\ge 0\,.
\]
We further assume the standard Adam-type condition
\begin{equation}\label{eq:beta_condition_sqrt}
\beta_1\le \sqrt{\beta_2}
\qquad\text{(or, }\,\beta_1^2\leq \beta_2\text{)}\,.
\end{equation}
We define the continuous bias-correction factor
\[
\eta(t):=\frac{\sqrt{1-\beta_2^{t/\alpha}}}{1-\beta_1^{t/\alpha}},
\qquad
\mu(t):=\frac{\eta(t)}{\sqrt{M_2(t)}+\varepsilon(t)},
\qquad 
F(t):=-\mu(t)\,M_1(t).
\]

\paragraph{Uniform bounds on $M_1$ and $\dot M_1$ on $[0,T]$:}
The variation-of-constants formula of equation \eqref{eq:M1_eq_diff} gives
\[
M_1(t)=\int_0^t \lambda_1 e^{-\lambda_1(t-s)}\,g(s)\,ds,
\]
hence $\|M_1(t)\|\le G_{\max}$ for all $t\in[0,T]$.
Moreover,
\[
\dot M_1(t)=\lambda_1\big(g(t)-M_1(t)\big).
\]
Using the identity
\[
g(t)-M_1(t)=e^{-\lambda_1 t}g(0)+\int_0^t e^{-\lambda_1(t-s)}\,\dot g(s)\,ds,
\]
we obtain for $t\ge \delta$,
\[
\|g(t)-M_1(t)\|
\le e^{-\lambda_1\delta}\|g(0)\|+\frac{1}{\lambda_1}\dot G_{\max}.
\]
Therefore,
\[
\|\dot M_1(t)\|
\le \lambda_1 e^{-\lambda_1\delta}\|g(0)\|+\dot G_{\max}.
\]
Since $\sup_{x\ge 0} x e^{-x\delta}=1/(e\delta)$, we have
\[
\lambda_1 e^{-\lambda_1\delta}\le \frac{1}{e\delta},
\]
and consequently,
\begin{equation}\label{eq:M1dot-bound-app}
\sup_{t\in[\delta,T]}\|\dot M_1(t)\|
\le \frac{\|g(0)\|}{e\delta}+\dot G_{\max}
=:C_{M,\delta,T},
\end{equation}
which is independent of $\alpha$.

\paragraph{Uniform bounds on $M_2$ and $\dot M_2$ on $[\delta,T]$:}
Let $r(t):=\|g(t)\|^2$. Since $g\in C^1([0,T])$ and $\|g\|\le G_{\max}$, we have
$r\in C^1([0,T])$ and
\[
|r(t)|\le G_{\max}^2,\qquad |\dot r(t)|\le 2G_{\max}\dot G_{\max}=: \dot R_{\max}.
\]
Applying the same argument as above to the scalar ODE for $M_2$ with source $r$
yields, for $t\ge\delta$,
\[
|\dot M_2(t)|
=\lambda_2|r(t)-M_2(t)|
\le \lambda_2 e^{-\lambda_2\delta}\,|r(0)|+\dot R_{\max}.
\]
Using again $\lambda_2 e^{-\lambda_2\delta}\le 1/(e\delta)$, we obtain
\begin{equation}\label{eq:vdot-bound-app}
\sup_{t\in[\delta,T]}|\dot M_2(t)|
\le \frac{|r(0)|}{e\delta}+\dot R_{\max}
=:C_{v,\delta,T},
\end{equation}
independent of $\alpha$.

\paragraph{Uniform bounds on $\eta$ and $\dot\eta$ on $[\delta,T]$:}
Write $\beta_a^{t/\alpha}=e^{-(|\ln\beta_a|/\alpha)t}$ and denote $c_a:=|\ln\beta_a|>0$. For $t\ge\delta$ we have $e^{-c_a t/\alpha}\le e^{-c_a\delta/\alpha}$, hence $\eta(t)$ is uniformly bounded on $[\delta,T]$, namely:
\[
\sup_{t\in[\delta,T]}|\eta(t)|\le C_{\eta,0}\,.
\]
For the derivative, direct differentiation gives a linear combination of terms of the form
\[
\frac{1}{\alpha}\,\beta_a^{t/\alpha}
=\frac{1}{tc_a}\Big(\frac{t c_a}{\alpha}e^{-t c_a/\alpha}\Big).
\]
Since $\sup_{x\ge 0}x e^{-x}=\frac{1}{e}$ and $t\ge\delta$, we obtain the uniform bound
\[
\frac{1}{\alpha}\,\beta_a^{t/\alpha}\le \frac{1}{e\,\delta\,c_a}\,,\qquad t\in[\delta,T],
\]
and therefore
\[
\sup_{t\in[\delta,T]}|\dot\eta(t)|\le C_{\eta,1},
\]
for some constant $C_{\eta,1}$ depending on $(\beta_1,\beta_2,\delta)$ but not on $\alpha$.

\paragraph{Uniform bound on $\mu$ on $[\delta,T]$:}
By $M_2(t)\ge 0$ and $\varepsilon(t)\ge \varepsilon_0$ we have
\[
\sqrt{M_2(t)}+\varepsilon(t)\ge \varepsilon_0>0.
\]
Thus $\mu$ is uniformly bounded:
\[
\sup_{t\in[\delta,T]}|\mu(t)|\le \frac{C_{\eta,0}}{\varepsilon_0}=:C_{\mu,0}.
\]

\paragraph{A moment-comparison estimate on finite horizons:}
Let $H^{(1)}_a(\tau):=\lambda_a e^{-\lambda_a\tau}$, $a\in\{1,2\}$. Then
\[
M_1(t)=\int_0^t H^{(1)}_1(\tau)\,g(t-\tau)\,d\tau,\qquad
M_2(t)=\int_0^t H^{(1)}_2(\tau)\,\|g(t-\tau)\|^2\,d\tau.
\]
For a fixed $T>0$ define
\[
C_{12}(T):=\sup_{t\in[0,T]}\int_0^t \frac{H^{(1)}_1(\tau)^2}{H^{(1)}_2(\tau)}\,d\tau.
\]
Then for all $t\in[0,T]$ one has
\begin{equation}\label{eq:M1_M2_comparison}
\|M_1(t)\|^2 \le C_{12}(T)\, M_2(t),
\qquad\text{and hence}\qquad
\frac{\|M_1(t)\|}{\sqrt{M_2(t)}}\le \sqrt{C_{12}(T)}\,.
\end{equation}
Indeed, by weighted Cauchy-Schwarz,
\[
\|M_1(t)\|^2
=\Big\|\int_0^t \frac{H^{(1)}_1(\tau)}{\sqrt{H^{(1)}_2(\tau)}}\Big(\sqrt{H^{(1)}_2(\tau)}\,g(t-\tau)\Big)\,d\tau\Big\|^2
\le \Big(\int_0^t \frac{H^{(1)}_1(\tau)^2}{H^{(1)}_2(\tau)}\,d\tau\Big)\,M_2(t),
\]
and taking the supremum over $t\in[0,T]$ yields \eqref{eq:M1_M2_comparison}. Moreover, since
\[
\frac{H^{(1)}_1(\tau)^2}{H^{(1)}_2(\tau)}
=\frac{\lambda_1^2}{\lambda_2}\,e^{-(2\lambda_1-\lambda_2)\tau},
\]
we obtain
\[
C_{12}(T)=
\begin{cases}
\frac{\lambda_1^2}{\lambda_2}\,\frac{1-e^{-(2\lambda_1-\lambda_2)T}}{2\lambda_1-\lambda_2}, & 2\lambda_1\neq\lambda_2,\\[6pt]
\frac{\lambda_1}{2}\,T, & 2\lambda_1=\lambda_2.
\end{cases}
\]
Under the standing assumption---\eqref{eq:beta_condition_sqrt}---we have
\[
1-\beta_2 \le 1-\beta_1^2 = (1-\beta_1)(1+\beta_1) < 2(1-\beta_1),
\]
hence $2\lambda_1-\lambda_2>0$ and therefore
\[
C_{12}(T)\le \int_0^\infty \frac{\lambda_1^2}{\lambda_2}e^{-(2\lambda_1-\lambda_2)\tau}\,d\tau
=\frac{\lambda_1^2}{\lambda_2(2\lambda_1-\lambda_2)}
=\frac{(1-\beta_1)^2}{(1-\beta_2)\,[\,1+\beta_2-2\beta_1\,]}
=:C_{12}^\star,
\]
which depends only on $(\beta_1,\beta_2)$ and is independent of $\alpha$.

\paragraph{Lipschitz bound for $F$ and the shift estimate:}
Write
\[
F(t)=-\frac{\eta(t)\,M_1(t)}{D(t)},
\qquad D(t):=\sqrt{M_2(t)}+\varepsilon(t)\ge \varepsilon_0.
\]
On $\{t:\,M_2(t)>0\}$ we can differentiate:
\[
\dot F(t)
= -\frac{\dot\eta\,M_1+\eta\,\dot M_1}{D}
+\eta\,M_1\;\frac{\dot D}{D^2},
\qquad
\dot D=\frac{\dot M_2}{2\sqrt{M_2}}+\dot\varepsilon.
\]
Notice that if $M_2(t)=0$, then \eqref{eq:M1_M2_comparison} implies $M_1(t)=0$, hence $F(t)=0$ and the bound is trivial at such times. Therefore, we bound each term on $[\delta,T]$ using $D\ge\varepsilon_0$, $\|M_1\|\le G_{\max}$, \eqref{eq:M1dot-bound-app}, \eqref{eq:vdot-bound-app},
and \eqref{eq:M1_M2_comparison} together with $C_{12}(T)\le C_{12}^\star$:
\begin{align*}
\Big\|\frac{\dot\eta\,M_1+\eta\,\dot M_1}{D}\Big\|
&\le \frac{C_{\eta,1}G_{\max}+C_{\eta,0}C_{M,\delta,T}}{\varepsilon_0},\\
\Big\|\eta\,M_1\frac{\dot\varepsilon}{D^2}\Big\|
&\le \frac{C_{\eta,0}G_{\max}\varepsilon_1}{\varepsilon_0^2},\\
\Big\|\eta\,M_1\frac{1}{D^2}\frac{\dot M_2}{2\sqrt{M_2}}\Big\|
&\le \frac{C_{\eta,0}}{2\varepsilon_0^2}
\Big(\sup_{t\in[\delta,T]}\frac{\|M_1(t)\|}{\sqrt{M_2(t)}}\Big)
\Big(\sup_{t\in[\delta,T]}|\dot M_2(t)|\Big)\\
&\le \frac{C_{\eta,0}}{2\varepsilon_0^2}\,\sqrt{C_{12}^\star}\,C_{v,\delta,T}.
\end{align*}
Therefore,
\[
\sup_{t\in[\delta,T]}\|\dot F(t)\|\le C_{\delta,T},
\]
where one may take
\[
C_{\delta,T}:=
\frac{C_{\eta,1}G_{\max}+C_{\eta,0}C_{M,\delta,T}}{\varepsilon_0}
+\frac{C_{\eta,0}G_{\max}\varepsilon_1}{\varepsilon_0^2}
+\frac{C_{\eta,0}}{2\varepsilon_0^2}\,\sqrt{C_{12}^\star}\,C_{v,\delta,T},
\]
which is independent of $\alpha$.

Finally, since $F$ is absolutely continuous on $[\delta,T]$ and $\|\dot F\|$ is bounded a.e., for any $t\in[\delta,T-\alpha]$ we have the shift estimate
\[
\|F(t+\alpha)-F(t)\|
\le \int_t^{t+\alpha}\|\dot F(s)\|\,ds
\le C_{\delta,T}\,\alpha.
\]

If the dynamics is written as
\[
\dot\theta(t)=F(t+\alpha)+R_\rho(t),\qquad \|R_\rho(t)\|\le C_R\,\rho,
\]
then the shift estimate implies that, after redefining
\[
\widetilde R_\rho(t):=R_\rho(t)+\left[F(t+\alpha)-F(t)\right],
\]
one has
\[
\dot\theta(t)=F(t)+\widetilde R_\rho(t),
\qquad 
\|\widetilde R_\rho(t)\|\le C_R\,\rho+C_{\delta,T}\,\alpha,
\]
uniformly for $t\in[\delta,T-\alpha]$.

\section{Appendix: Control of the tracking error $e(t):=g(t)-M_1(t)$}\label{app:tracking-error}
Let $M_1$ be defined by
\[
\dot M_1(t)+\lambda_1 M_1(t)=\lambda_1 g(t),
\qquad M_1(0)=0,
\]
and define the tracking error
\[
e(t):=g(t)-M_1(t).
\]
Then
\[
\dot e(t)=\dot g(t)-\dot M_1(t)
=\dot g(t)-\lambda_1\big(g(t)-M_1(t)\big)
=\dot g(t)-\lambda_1 e(t),
\]
hence $e$ satisfies the linear ODE
\[
\dot e(t)+\lambda_1 e(t)=\dot g(t),
\qquad e(0)=g(0)-M_1(0)=g(0).
\]
By variation of constants,
\[
e(t)=e^{-\lambda_1 t}\,g(0)+\int_0^t e^{-\lambda_1 (t-s)}\,\dot g(s)\,ds.
\]
Therefore, for any $t\ge 0$,
\[
\begin{aligned}
\|g(t)-M_1(t)\|
&
\le e^{-\lambda_1 t}\,\|g(0)\|
+\int_0^t e^{-\lambda_1 (t-s)}\,\|\dot g(s)\|\,ds\\
&\le e^{-\lambda_1 t}\,\|g(0)\|
+\frac{1}{\lambda_1}\sup_{s\in[0,t]}\|\dot g(s)\|. 
\end{aligned}
\]
Since $\lambda_1=(1-\beta_1)/\alpha$, we have
\[
\frac{1}{\lambda_1}=\frac{\alpha}{1-\beta_1}=O(\alpha).
\]
In particular, for any fixed $\delta>0$ and all $t\in[\delta,T-\alpha]$,
\[
\|g(t)-M_1(t)\|
\le e^{-\lambda_1\delta}\,\|g(0)\|
+\frac{\alpha}{1-\beta_1}\,\|\dot g\|_{L^\infty([\delta,T-\alpha])}
\;\le\; e^{-\lambda_1\delta}\,\|g(0)\|+C_\delta\,\frac{\alpha}{1-\beta_1},
\]
provided $\|\dot g\|_{L^\infty([\delta,T-\alpha])}<\infty$. In particular, for fixed $\delta>0$ one has $e^{-(1-\beta_1)\delta/\alpha}=o(\alpha)$ as $\alpha\to 0$, so $\sup_{t\in[\delta,T-\alpha]}\|g(t)-M_1(t)\|=O(\alpha/(1-\beta_1))$.

\section{Appendix: Justification of the key step in Proposition~5}\label{App:Prop5}

In this appendix we justify in detail the step
\begin{align*}
 \tau \int_{\mathbb{R}} \mathrm{d}\xi \,
 \theta^i(\xi)\,e^{\frac{2\xi}{\alpha}}\,a(\xi)\,
 \frac{\delta F_i(\tau T_\xi\theta, \xi)}{\delta \theta^j(\sigma + t)}
 &=
 \tau\,e^{\frac{2(\sigma + t)}{\alpha}}\,a(\sigma + t)
 \int_{\mathbb{R}} \mathrm{d}\xi \,
 \theta^i(\xi)\,
 \frac{\delta F_j(\tau T_{\sigma + t} \theta,\sigma+t)}{\delta \theta^i(\xi)} \\
 &= \tau\,e^{\frac{2(\sigma + t)}{\alpha}} a(\sigma + t)\,
 \frac{\mathrm{d}}{\mathrm{d}\tau}
 \Bigl[F_j(\tau T_{\sigma + t}\theta, \sigma + t)\Bigr],
\end{align*}
used in the proof of Proposition~\ref{prop:nonlocal-lagrangian}.

\subsection*{1. Use of the reciprocity condition}
Recall the reciprocity condition in Proposition~\ref{prop:nonlocal-lagrangian}: for all trajectories
$\theta \in K$ and all $\xi,\sigma \in \mathbb{R}$,
\[
e^{\frac{2\xi}{\alpha}}\,a(\xi)
\frac{\delta F_i(T_\xi\theta, \xi)}{\delta \theta^j(\sigma)}
=
e^{\frac{2\sigma}{\alpha}}\,a(\sigma)
\frac{\delta F_j(T_\sigma\theta, \sigma)}{\delta \theta^i(\xi)}\,.
\]
This identity is assumed to hold for every trajectory in the kinematic space. In particular, it holds for the rescaled trajectory $\varphi = \tau \theta\,.$ Applying the last equation with this rescaled trajectory and the time arguments
$\xi$ and $\sigma + t$, we obtain
\begin{equation}
e^{\frac{2\xi}{\alpha}}\,a(\xi)
\frac{\delta F_i(T_\xi(\tau \theta), \xi)}{\delta (\tau\theta)^j(\sigma + t)}
=
e^{\frac{2(\sigma + t)}{\alpha}}\,a(\sigma + t)
\frac{\delta F_j(T_{\sigma + t}(\tau \theta), \sigma + t)}{\delta (\tau\theta)^i(\xi)}.
\label{eq:appendix-pre-chain}
\end{equation}
Using the linearity of the translation operator $T_\xi$,
\[
T_\xi(\tau\theta) = \tau\,T_\xi\theta, \qquad
T_{\sigma+t}(\tau\theta) = \tau\,T_{\sigma+t}\theta,
\]
\eqref{eq:appendix-pre-chain} becomes
\begin{equation}
e^{\frac{2\xi}{\alpha}} \alpha(\xi)
\frac{\delta T_i(\tau T_\xi\theta, \xi)}{\delta (\tau\theta)^j(\sigma + t)}
=
e^{\frac{2(\sigma + t)}{\alpha}} \alpha(\sigma + t)
\frac{\delta T_j(\tau T_{\sigma + t}\theta, \sigma + t)}{\delta (\tau\theta)^i(\xi)}.
\label{eq:appendix-pre-chain-2}
\end{equation}
We now apply the chain rule for the linear rescaling $\varphi^j(s) = \tau\,\theta^j(s)$. A variation $\delta\theta$ induces the variation $\delta\varphi^j(s) = \tau\,\delta\theta^j(s)$, and therefore
\[
\frac{\delta}{\delta(\tau\theta)^j(s)}
= \frac{1}{\tau}\,\frac{\delta}{\delta\theta^j(s)}.
\]
Using this in \eqref{eq:appendix-pre-chain-2}, we obtain
\begin{equation}
e^{\frac{2\xi}{\alpha}}\,a(\xi)
\frac{\delta F_i(\tau T_\xi\theta, \xi)}{\delta \theta^j(\sigma + t)}
=
e^{\frac{2(\sigma + t)}{\alpha}}\,a(\sigma + t)
\frac{\delta F_j(\tau T_{\sigma + t}\theta, \sigma + t)}{\delta \theta^i(\xi)}.
\label{eq:appendix-sym-applied}
\end{equation}
Multiplying both sides of \eqref{eq:appendix-sym-applied} by
$\tau\,\theta^i(\xi)$ and integrating over $\xi \in \mathbb{R}$ yields
\[
\begin{aligned}
 \tau \int_{\mathbb{R}} \mathrm{d}\xi \,
 \theta^i(\xi)\,e^{\frac{2\xi}{\alpha}}\,a(\xi)\,
 \frac{\delta F_i(\tau T_\xi\theta, \xi)}{\delta \theta^j(\sigma + t)}
 &=
 \tau\,e^{\frac{2(\sigma + t)}{\alpha}}\,a(\sigma + t)
 \int_{\mathbb{R}} \mathrm{d}\xi \,
 \theta^i(\xi)\,
 \frac{\delta F_j(\tau T_{\sigma + t} \theta,\sigma+t)}{\delta \theta^i(\xi)}.
\end{aligned}
\]
This is precisely the first equality used in the main text.

\subsection*{2. From the integral to the derivative
$\frac{\mathrm{d}}{\mathrm{d}\tau}F_j(\tau T_{\sigma+t}\theta,\sigma+t)$}
We now justify the second equality
\[
\int_{\mathbb{R}} \mathrm{d}\xi \,
 \theta^i(\xi)\,
 \frac{\delta T_j(\tau T_{\sigma + t} \theta,\sigma+t)}{\delta \theta^i(\xi)}
 =
 \frac{\mathrm{d}}{\mathrm{d}\tau}
 \Bigl[T_j(\tau T_{\sigma + t}\theta, \sigma + t)\Bigr].
\]
Fix $t$ and $\sigma$, and define the functional $F_j[\theta] := F_j(T_{\sigma + t}\theta, \sigma + t).$ For each $\tau$, we evaluate $F$ along the rescaled trajectory
\[
\theta_\tau := \tau \theta, \qquad
F_j[\theta_\tau] = F_j(\tau T_{\sigma + t}\theta, \sigma + t).
\]
The derivative of $F_j[\theta_\tau]$ with respect to $\tau$ is
\begin{equation}
\frac{\mathrm{d}}{\mathrm{d}\tau} F_j(\tau T_{\sigma + t}\theta, \sigma + t)
=
\frac{\mathrm{d}}{\mathrm{d}\tau} F_j[\theta_\tau]
=
\int_{\mathbb{R}} \mathrm{d}\xi \,
\frac{\delta F_j}{\delta\theta^i(\xi)}\bigg|_{\theta=\theta_\tau}
\frac{\partial\theta_\tau^i(\xi)}{\partial\tau}.
\label{eq:appendix-gateaux}
\end{equation}
Since $\theta_\tau^i(\xi) = \tau\,\theta^i(\xi)$, we have
\[
\frac{\partial\theta_\tau^i(\xi)}{\partial\tau} = \theta^i(\xi),
\]
and by the definition of the functional derivative of $F$,
\[
\frac{\delta F_j}{\delta\theta^i(\xi)}\bigg|_{\theta=\theta_\tau}
=
\frac{\delta F_j(\tau T_{\sigma + t}\theta, \sigma + t)}{\delta\theta^i(\xi)}.
\]
Substituting these expressions into \eqref{eq:appendix-gateaux} gives
\begin{equation}
\frac{\mathrm{d}}{\mathrm{d}\tau} F_j(\tau T_{\sigma + t}\theta, \sigma + t)
=
\int_{\mathbb{R}} \mathrm{d}\xi \,
\theta^i(\xi)\,
\frac{\delta F_j(\tau T_{\sigma + t}\theta, \sigma + t)}{\delta\theta^i(\xi)}.
\label{eq:appendix-integral-derivative}
\end{equation}
This is precisely the second equality used in the main text.

Notice that the only regularity assumption needed is that $F_j$ be functionally differentiable with respect to the trajectory $\theta$, a natural requirement in this setting.

\bibliography{iclr_conference}
\bibliographystyle{iclr_conference}

\end{document}